\pdfoutput=1
\documentclass{article}

\PassOptionsToPackage{numbers,compress}{natbib}
\usepackage[preprint]{neurips_2026}

\usepackage[utf8]{inputenc}
\usepackage[T1]{fontenc}
\usepackage{courier}
\usepackage[table,xcdraw,usenames,dvipsnames]{xcolor}
\usepackage{hyperref}
\usepackage{url}
\usepackage{booktabs}
\usepackage{amsfonts}
\usepackage{nicefrac}
\usepackage{microtype}
\usepackage{graphicx}
\usepackage{subcaption}
\usepackage{pifont}
\usepackage{dsfont}
\usepackage{enumitem}
\usepackage{wrapfig}
\usepackage{tabularx}
\usepackage{xspace}
\usepackage[most]{tcolorbox}
\usepackage{makecell}
\usepackage{multirow}
\usepackage{amsmath}
\usepackage{amssymb}
\usepackage{mathtools}
\usepackage{amsthm}
\usepackage[capitalize,noabbrev]{cleveref}

\usepackage{amsmath,amsfonts,bm}

\def\eqref#1{equation~\ref{#1}}
\def\1{\bm{1}}

\def\mA{{\bm{A}}}

\def\mH{{\bm{H}}}

\def\mO{{\bm{O}}}

\def\mZ{{\bm{Z}}}

\DeclareMathAlphabet{\mathsfit}{\encodingdefault}{\sfdefault}{m}{sl}
\SetMathAlphabet{\mathsfit}{bold}{\encodingdefault}{\sfdefault}{bx}{n}

\newcommand{\stitle}[1]{\vspace{1ex}\noindent{\bf #1}}

\newcommand{\ms}[2]{{#1\tiny{$\pm$#2}}}

\definecolor{mygrey}{gray}{0.4}

\newcommand{\pipeline}{\textit{Trajectory Graph Copilot}}
\newcommand{\ours}{\textit{Graph Debugger}}

\newcommand{\alfw}{AlfWorld}
\newcommand{\sciw}{ScienceWorld}
\newcommand{\textw}{TextWorld}
\newcommand{\trapl}{TravelPlanner}

\theoremstyle{plain}
\newtheorem{theorem}{Theorem}[section]

\theoremstyle{definition}

\newtheorem{assumption}[theorem]{Assumption}
\theoremstyle{remark}

\title{Leveraging Trajectory Graphs for Pre-Execution Error Diagnosis in Agentic LLM Systems}

\author{%
\textbf{Xu Zheng}$^{1}$ \quad
\textbf{Zhuomin Chen}$^{1}$ \quad
\textbf{Chaohao Lin}$^{1}$ \quad
\textbf{Hua Wei}$^{2}$ \\ 
\textbf{Haifeng Chen}$^{3}$ \quad
\textbf{Wei Cheng}$^{3}$ \quad
\textbf{Dongsheng Luo}$^{4}$\thanks{Corresponding author. This work was primarily conducted while he was at Florida International University.}\\
$^{1}$Florida International University, Miami, FL, USA\\
$^{2}$Arizona State University, Tempe, AZ, USA\\
$^{3}$NEC Laboratories America, Princeton, NJ, USA\\
$^{4}$Singapore Management University, Singapore\\
}

\begin{document}

\maketitle

\begin{abstract}
Large Language Model~(LLM)-based agents have demonstrated exceptional performance across a wide range of complex interactive tasks. However, they often struggle with long-horizon interactive tasks common in domains, such as embodied AI. The complexity and vast action spaces in these settings lead to compounding errors, where a single suboptimal action can derail an entire trajectory, causing the agent to exhaust its limited step budget on inefficient or unrecoverable paths. To overcome this without costly fine-tuning, we draw inspiration from software debugging, where execution logs are analyzed to preemptively catch errors. We propose {\pipeline}, a novel framework that acts as a ``copilot'' for LLM agents by diagnosing potential action errors before they are executed. At its core, {\ours} models historical trajectories as a probabilistic graph and uses a Graph Neural Network to identify sequential action patterns that frequently lead to failure. Functioning as a proactive diagnostic sandbox, our method provides early warnings on potentially flawed actions, prompting the agent to self-correct. This pre-action error diagnosis prevents costly mistakes, significantly enhancing the agent's ability to complete long-horizon tasks successfully.
The extensive experiments on four benchmarks with three LLM agents demonstrate a $14.69\%$ pass ratio improvement on average.
\end{abstract}

\section{Introduction}
Large Language Models (LLMs), such as ChatGPT~\cite{chatgpt}, Gemini~\cite{team2024gemini}, and Llama~\cite{touvron2023llama}, possess a remarkable capacity language comprehension and generation.
When equipped with tools~\cite{yang2023gpt4tools,wu2024autogen},
these LLMs become powerful agents capable of extraordinary performance in complex applications such as coding~\cite{islam2024mapcoder,qian2023chatdev,zhang2024codeagent}, scientific
reasoning~\cite{wang2022scienceworld}, and embodied artificial intelligence~(Embodied AI)~\cite{puig2018virtualhome,shridhar2020alfworld,ma2024survey}. These tasks often require agents to plan and execute long sequences of actions while interacting with an environment~\cite{li2022pre,xiong2024watch,li2024embodied,yang2025embodiedbench}. 

Despite their powerful reasoning capabilities, LLM agents often falter in complex and unfamiliar environments due to compounding errors~\cite{wang2024e2cl,xie2024revealing}. Prevailing strategies attempt to mitigate this through post-hoc learning, using methods like environmental data exploration~\cite{xiang2023language,xie2024revealing,song2024trial} or refinement learning~\cite{xiong2024watch,yuan2025agent,wang2025steca}. However, these approaches share a fundamental limitation. 
By learning primarily from the final outcomes of entire trajectories, they struggle to pinpoint the specific, step-level actions that lead to failure. % A binary success/failure signal
A reward for a trajectory provides a much weaker learning signal than feedback explaining why a particular action was incorrect (e.g., targeting a non-existent object or performing an invalid sequence). Without this granular, causal feedback, agents learn inefficiently from sparse signals and are prone to repeating similar mistakes.

To address this gap, we shift the paradigm from post-hoc trajectory analysis to proactive, step-level error diagnosis. We draw inspiration from software debugging. When a program fails, a developer uses a debugger to trace the execution, inspect the context, and pinpoint the exact line of code that caused the error. This pre-action, fine-grained analysis is far more effective than simply observing that the program crashed. This inspires our framework, {\pipeline}, which incorporates a graph-based diagnostic module, {\ours}. It acts as a ``debugger'' for the LLM agent, analyzing its intended action in the context of the recent past to flag potential errors before they are executed. 
This allows the agent to find potential errors early, preventing the accumulation of costly mistakes that would otherwise doom the entire task.

By framing agent execution within the Partially Observable Markov Decision Process (POMDP) paradigm~\cite{xiong2024watch,wang2025steca}, our method \ding{182} \textit{transforms historical trajectories into a probabilistic graph}. This structure encodes domain knowledge by capturing the underlying relationships between actions and observations~\cite{bishop2006pattern,koller2009probabilistic}. In our formulation, nodes represent actions and edges encode observations, which can be viewed as an adaptation of the classical state–transition diagram. Unlike the traditional methods, such as the retrieval-based methods~\cite{zhou2024trad}, our graph-based approach offers two key advantages: it \ding{183} \textit{provides higher performance} 
and \ding{184} \textit{requires fewer samples} to achieve the same level error detection rate. 
We detail these insights in Section~\ref{sec:theory}. Furthermore, instead of directly suggesting correct actions, {\ours} can \ding{185} \textit{serve as a diagnosis sandbox for the LLM-Agent decision making}. By providing only potential error warnings, our method encourages LLM agents to rely on their own reasoning. This minimizes the bias introduced by fine-tuning on a fixed dataset, ultimately leading to better generalization on new tasks.

To empirically verify the effectiveness of our framework, we conduct experiments on four benchmarks, including embodied AI environments and planning tasks. We begin by collecting trajectory datasets and annotating actions with corresponding labels, followed by conducting detection experiments on them. Compared with traditional text classification and LLM-based methods, {\ours} demonstrates a consistent advantage in step-level error detection. Moreover, we apply {\ours} as a diagnosis sandbox to provide error feedback in advance. By leveraging In-Context Learning (ICL) with feedback information, our method enhances LLM agents' pass ratio by $14.69\%$ on average, outperforming baselines. In summary, the contributions are as follows:
  \begin{itemize}[itemsep=1.5pt,topsep=0pt,parsep=0pt,leftmargin=*]%[leftmargin=*]
  \vspace{-0.1cm} 
  \item[\ding{72}] We introduce, {\ours}, a novel method based on a probabilistic graph model, to address the challenge of step-level action error detection in LLM agents.
  \item[\ding{72}] Our framework, {\pipeline}, integrates this error detection module as a distinct entity. By serving as a graph-based diagnostic module for LLM agents, our approach enhances their performance by providing real-time error warnings.
  \item[\ding{72}] Experiments conducted across multiple environments and LLM agents confirm the superiority of {\ours} in error detection and validate the effectiveness of the overall framework. 
  
  \end{itemize}

\section{Preliminaries}
\label{sec:preliminary}

\stitle{Long-Horizon Tasks as POMDPs.}
Following prior work, we model an LLM agent's interaction in a long-horizon task as a POMDP~\citep{pmlr-v202-carta23a,wang2025steca,song2024trial,xiong2024watch}. A POMDP is defined by the tuple $\mathcal{M} = (\mathcal{G},\mathcal{S},\mathcal{A},\mathcal{J},\mathcal{R},\mathcal{O},\gamma)$, where $\mathcal{G}$ is the goal space, $\mathcal{S}$ is the state space, $\mathcal{A}$ is the action space, $\mathcal{J}:\mathcal{S}\times\mathcal{A} \rightarrow \mathcal{S}$ is the state-transition function, $\mathcal{R}:\mathcal{S}\times\mathcal{A}\times\mathcal{G}\rightarrow \mathbb{R}$ is the reward function, $\mathcal{O}$ is the observation space, and $\gamma$ is the discount factor. In the text-based environments considered by this work, the spaces $\mathcal{G}$, $\mathcal{A}$, and $\mathcal{O}$ are subsets of natural language.

\stitle{Graphical Models for Sequential Decision Making.}
A Markov Decision Process can be conceptually viewed as a Probabilistic Graphical Model (PGM), where nodes represent states and edges represent transitions. This graphical perspective highlights the sequential dependencies inherent in agent trajectories. However, in a POMDP, the true state is latent, and the belief state (a probability distribution over states) is often high-dimensional and intractable to model explicitly, especially with language-based observations. Directly constructing and reasoning over a formal state-based PGM is therefore impractical. This motivates our work to develop a different, more practical graphical representation learned directly from trajectory data to diagnose action errors.

\stitle{Problem Formulation: Step-Level Error Diagnosis.}
\label{sec:problem}
Given a particular goal $g \in \mathcal{G}$, an LLM agent generates a trajectory history of alternating observations and actions, $(o_0, a_0, \dots, o_{t-1}, a_{t-1}, o_t)$. Before executing the next proposed action $a_t$, our goal is to determine if this action is productive or a potential error. To formalize what constitutes a ``good'' action, we move beyond simple binary 
success/failure signals. {Inspired by the benchmark AgentBoard~\citep{chang2024agentboard}, where the task is decomposed into subgoals, we conceptualize complex tasks as requiring the completion of several ordered milestones.} For example, in AlfWorld, the task ``clean a plate and put it on countertop'' involves a sequence of milestones, including collecting the plate, cleaning it, and placing it correctly. The quality of an action can therefore be evaluated based on the progress it makes toward the next milestone.  Figure~\ref{fig:pomdp} illustrates this concept by comparing an ideal Expert Trajectory, which progresses smoothly through milestone states, with an Agent Trajectory, which may deviate and follow a less optimal path.

To create an informative label space $\mathcal{Y}=\{1,...,C\}$, we define a set of fine-grained error categories that capture common failure modes in long-horizon tasks. These categories are not only inspired by task planning principles, using milestones to identify errors like \textit{Precondition Not Met (P.M.)} or \textit{Condition Met, Action Not Taken (C.M.)}, but also include general procedural errors such as \textit{Repeated Action (R.A.)} and \textit{Illegal Action (I.A.)}. This multi-faceted approach provides a rich, step-level diagnostic signal. The complete set of six error definitions is detailed in Appendix~\ref{sec:label_definition}. Our objective is to learn a diagnostic function $f$ that maps a trajectory history to a label for the \textit{current} proposed action:
$f: (o_0, a_0, \dots, o_t, a_t) \mapsto y_t, \quad \text{where } y_t \in \mathcal{Y}$.
This provides granular, step-level feedback that is far more informative than a sparse, trajectory-level reward.

\section{{\pipeline}}
In this section, we introduce a novel framework, {\pipeline} for LLM agents, which includes a graph-based error diagnosis module, {\ours}.
Initially, we transform text trajectories into graphs from the perspective of PGM. Subsequently, we adopt the idea of TextGCN~\citep{yao2019graph} 
\begin{wrapfigure}{r}{0.40\textwidth}
  \vspace{-0.3cm}
  \centering
  \includegraphics[width=0.40\textwidth]{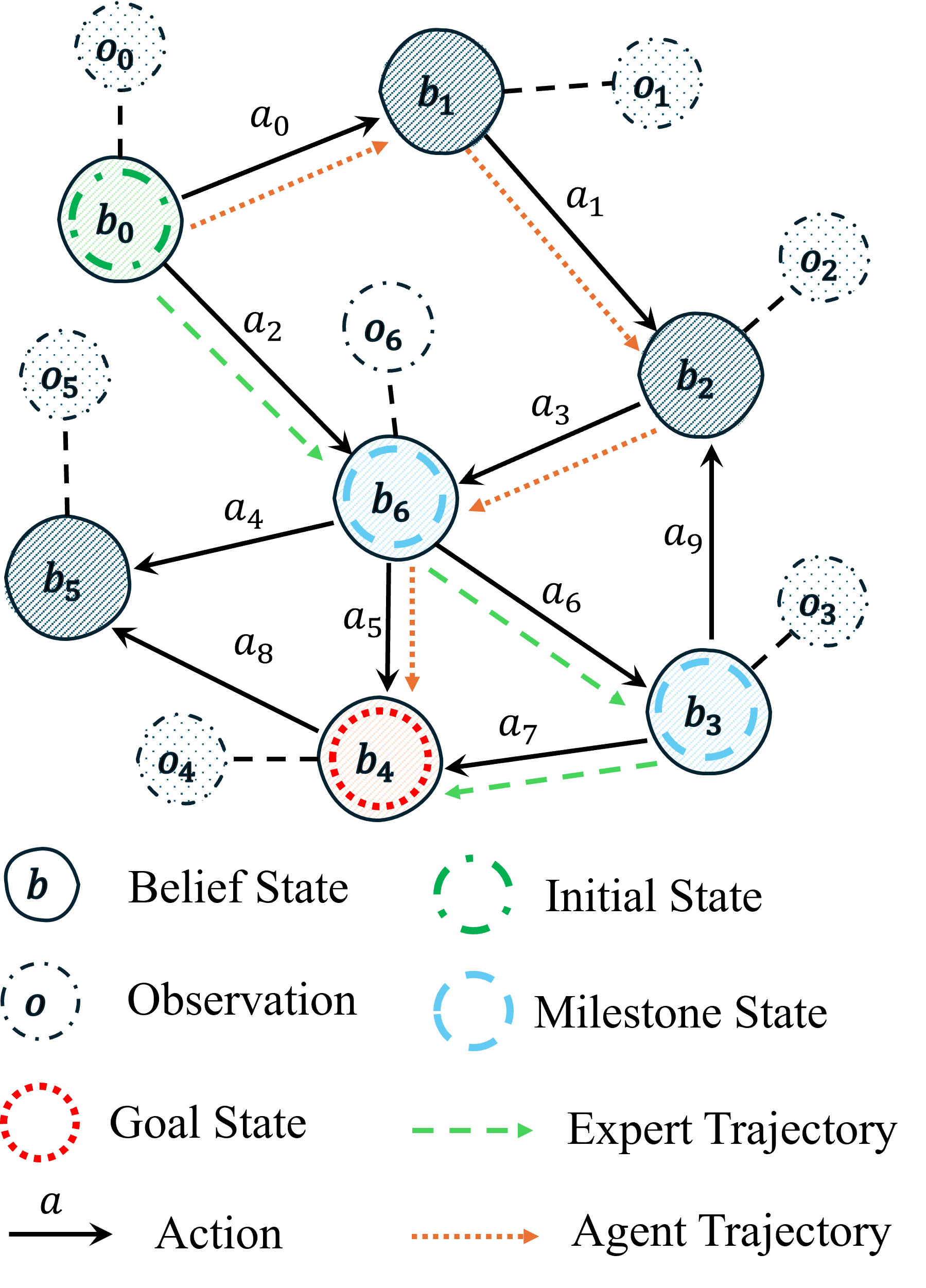}   
  \vspace{-0.4cm}
  \caption{A conceptual diagram of a POMDP illustrating our error diagnosis task. An expert follows an optimal trajectory from the Initial State ($b_0$) to the Goal State ($b_4$) by progressing through Milestone States ($b_6, b_3$). In contrast, the agent's trajectory is suboptimal.}
  \label{fig:pomdp}
  \vspace{-1.2cm}
\end{wrapfigure}
to regard the action error diagnosis as a node classification task. Finally, we apply {\ours} as a diagnosis sandbox in LLM agents for step-level action debugging and providing feedback for decision making. The overall framework is shown in Figure~\ref{fig:framework}.

\subsection{Graph Construction}
Building an accurate graph based on PGM presents two challenges: representing the dependency between states and actions, and accurately representing states in the graph. From Section~\ref{sec:preliminary}, the LLM agents' trajectories can be expressed as paths in a state transition graph. However, representing actions as edges duplicates the graph structure, reducing the overall information content, since different states may share the same action. 
To better extract the dependency between states and actions, a heterogeneous graph~\citep{sutton1998reinforcement}, where the nodes $\mathcal{V}$ comprise both states $\mathcal{B}$ and actions $\mathcal{A}$, offers a more effective and structured representation. {The heterogeneous graph enables node reduction by merging states with highly similar neighbors into a supernode, yielding a more generalized representation. However, due to the partial observation, it is hard to estimate the state in the graph.}
To address this issue, we consider using observations instead of states. In a POMDP, observations follow the conditional probability $p(o|s)$. Given the observation posterior probability $q(o)$ and $p(o|s)$, the state distribution can be inferred by $p(s) \varpropto \sum_{o}p(o|s)q(o)/\sum_{s'}p(o|s')$. Therefore, states are learned as implicit knowledge through observations. However, given that the observation results are in natural language and cannot be easily discretized, we integrate them as attributes within the edges to form an action-centric graph.

During implementation, to obtain a robust representation of nodes, we utilize a natural language processing tool, NLTK~\citep{bird2009natural}, {to finish some regular preprocessing, such as removing meaningless words/numbers.}  We then deduplicate actions to form a set of unique nodes and construct the PGM-based graph by linking them according to their order in the trajectories. Finally, we reform the state-transition diagram as an action-centric graph $G=(\mathcal{V},\mathcal{E})$, where $\mathcal{V} \subseteq \mathcal{A}$ and $\mathcal{E} \subseteq \mathcal{O}$. 

\begin{figure*}[t]
\centering
\includegraphics[width=1.00\textwidth]{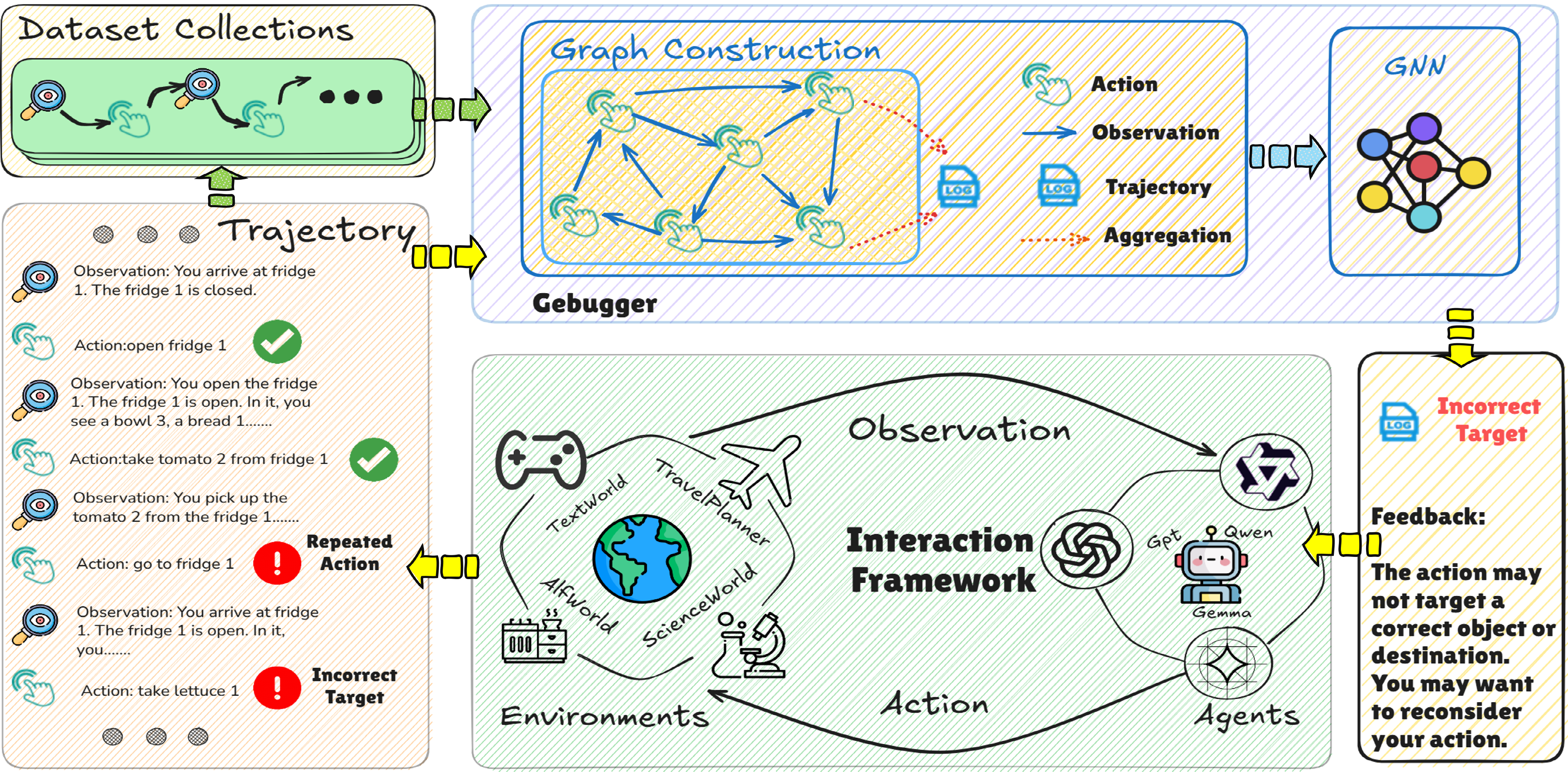}
\caption{Overall pipeline of {\pipeline}. We first collect the trajectories. {We then use the {\ours} to predict the action error labels, which includes converting trajectories into a PGM-based graph and utilizing the GNN-based detector. Finally, our framework provides external information for agent decision-making.}}
\label{fig:framework}
% \vspace{-0.5cm}
\end{figure*}

\subsection{Potential Error Diagnosis} 
To achieve a step-level action error detection, we adopt the TextGCN~\citep{yao2019graph} to regard trajectories as sentences. 
We use the pretrained Bert~\citep {devlin2019bert} model to initialize action and observation embeddings. By linking the trajectory nodes with action nodes, the error probe task is converted into a trajectory node classification task. Formally, the complete graph adjacency matrix $A$ is defined as:
\begin{align}
\label{eq:attr}
A_{ij} = \left \{ \begin{array}{cc}
1 & v_i,v_j \in \mathcal{A}, \text{and}~ [v_i,o,v_j] \in \xi  \\
1 & v_i \in \mathcal{A}, \text{and}~ v_j \in \mathcal{T}  \\
0 &  \text{otherwise}\\
\end{array} \right. ,
\end{align}
where $[v_i,o,v_j]$ is a subsequence of a trajectory $\xi$, {$\mathcal{T}$ is the set of trajectories} and $o$ is an observation. For edge attribution, we collect all the observations $\{o, o \in [a_i,o,a_j] \}$ linking the same action pair and get the average of the Bert embedding. The rest of the nodes and edges are set up with all zeros.

For potential error detection, we implement the detection mechanism outlined in the pipeline through a Graph Neural Network~(GNN) architecture. Given a trajectory $\xi$, our goal is to train a mapping function $f$ that probes if there is a potential error. This corresponds to the operation of the GNN detector across $K$ rounds of message passing. Each GNN layer performs one round of message passing, defined by the following update rule:
\begin{align*}
\label{eq:gnnupdate}
a_v^{(l)} &= \text{AGG}^{(l)}\left({h_u^{(l-1)}: u\in \mathcal{N}(v)}\right), \\
h_v^{(l)}  &= \text{COMBINE}^{(l)}\left(h_v^{(l-1)},a_v^{(l)}\right),
\end{align*}
where $a_v^{(l)}$ denotes the aggregated message at $l$-th layer, $h_v^{(l)}$ the feature vector of node $v$, and $\mathcal{N}(v)$ its set of neighbors. The $\text{AGG}^{(l)}$ is a function that aggregates information from neighboring nodes, while $\text{COMBINE}^{(l)}$ updates the node representations,{ following the definition in previous work~\citep{xu2018how}.} After $K$ rounds of updating, this process yields the final node embeddings $\mH$. For the error detection, we apply a softmax function to the final embeddings to obtain probabilities, $\mZ = \text{softmax}(\mH)$. The GNN model is trained by minimizing the cross-entropy loss on the labeled trajectories:
\begin{equation}
\mathcal{L} = - \sum_{t \in \mathcal{T}}\sum_{\ell \in \mathcal{Y}} Y_{t\ell}\ln Z_{t\ell},
\end{equation}
where $\mathcal{T}$ represents the train set of trajectories and $\mathcal{Y}$ is the label space. During testing, we inject the trajectory node into the existing graph constructed by the training set. To address potential issues where an action, observation, or task falls outside the defined scope, we use embeddings to search for and retrieve the corresponding trajectory path from the existing graph. We provide the detailed information in Appendix~\ref{sec:injection}.

\subsection{In-Context Learning Feedback}
To fully leverage the corrective signals provided by our external feedback module, we design a mechanism that integrates these signals into the agent's reasoning loop through ICL~\citep{liskavets2025prompt,zhang2022active,zhou2022large}. Rather than fine-tuning the underlying language model, our approach dynamically adapts the agent's behavior during inference by augmenting the prompt with structured feedback information. The detailed implementation is available in Appendix~\ref{sec:label_definition}.

Without feedback, the agent's generative behavior is modeled by the conditional distribution $P_{agent}(A_{i+1}|(O_0,A_0,...O_{i},A_{i}))$. With the environment copilot feedback, it first evaluates the $A_{i+1}$ initialized by the agent and generates the feedback signal $Y_{i}$, where the signal space is defined in Section~\ref{sec:preliminary}. If $Y_{i}$ indicate there is potential error, then the agent should regenerate $A'_{i+1}$ according to the revised conditional distribution $P_{agent}(A'_{i+1}|(O_0,A_0,...O_{i}, A_{i}), A_{i+1}, Y_{i})$. The revised conditional distribution can be viewed as the agent's posterior over actions. In practice, ICL operationalizes this posterior update by conditioning the model on explicit feedback text in the prompt.

\section{Theoretical Analysis}
\label{sec:theory}
One goal of this work is to detect the potential error based on the trajectory. A straightforward approach is to directly retrieve historical trajectories to determine if the current action is erroneous. A key challenge is managing the rapidly growing size of the trajectories as the number of interactions with the environment increases. 
Alternatively, because the action and observation modalities are consistent, the problem of error detection can be formulated as a text sequence classification task.
Prior works~\citep{taha2024comprehensive,li2022survey} have explored sequence-based modeling approaches for the task, such as Bert-based detectors~\citep{devlin2019bert}. 
Despite the promise, the structured information hidden in the trajectories is not well-explored. Drawing inspiration from PGM, we formulate the action error probe task as a node classification on a converted graph. Building upon this foundation, we show that graph-based methods have a lower Bayes risk and sample complexity than sequence-based methods under the same generalization error. 

To better serve theory analysis, we reorganize the trajectory as a tuple $X=(\tau, G, \mO_{1:k}, \mA_{1:k})$, where $G$ is the task, $\tau$ is the trajectory identifier, $O_i, A_i$ are the $i$-th observation and action for $i=1,...k$, and $k$ is the number of action/observation in trajectory. The decision target is the ground truth class $Y \in \mathcal{Y}$ of the final action.
We consider a class of baselines, the \textit{empirical sequence-based} (ESB) method, which captures the critical characteristics of existing approaches. The classical ESB method, such as a fine-tuned Bert classification model, contains two modules, a sequence representation mapping $U=\Phi_{seq}(X) \in \mathcal{U} \subseteq \mathbb{R}^{d_{seq}}$ and a classifier $h_{seq}: \mathcal{U} \to \mathcal{Y}$. 
In this work, the graph representation is $S=\Phi_{g}(X)\in \mathcal{S} \subseteq \mathbb{R}^{d_g}$ obtained by \textit{probabilistic graph-based} (PGB) approach, where $S$ is chosen to be the Markov blanket~\citep{pearl1998graphical} of $Y$ in the graph induced by the connectivity rules. 
Similarly, a classifier $h_{g}$ is learned from $(S, Y)$ pairs. To better compare ESB and PGB, 
we first make two conventional assumptions as follows.
\begin{assumption}[\textbf{Conditional Sufficiency}]
\label{assump:A1}
There exists a representation \(S=\mathrm{MB}_G(Y)\) (the Markov blanket of \(Y\) under the graph construction) such that $ Y \ \perp\!\!\!\perp\  X\setminus S \ \big| \ S $. 
\end{assumption}
This assumption indicates that the label of an action is determined in limited steps and dependent on other steps, which is aligned with the Markov decision process. {In the Appendix~\ref{sec:label_definition}, we provide an example to illustrate why this assumption holds.} In Section~\ref{sec:problem}, we follow the assumption to use the milestone states to refer to the action labels.  

\begin{theorem}[\textbf{Bayes Risk Ordering and Sample Complexity}]
    \label{th:1}
For any measurable classifier $h$ and representation $\Phi$, define the risk as $R(h \circ \Phi) \coloneqq \mathbb{P}(h(\Phi(X)) \neq Y)$, and let $R^*(\cdot)$ denote the corresponding Bayes (irreducible) risk. Among all representations, the sufficient statistic $S$ is optimal: its Bayes risk equals that of the full input $X$. For any other representation $U = \Phi(X)$, we have
\begin{equation}
    R^*(S) \le R^*(U),
\end{equation}
where the inequality is strict whenever $I(Y;S) > I(Y;U)$. This ordering of risks translates directly to sample complexity: to achieve classification error $\epsilon$ under the same error tolerance, the graph and sequential representations require $m_{\mathrm{g}}$ and $m_{\mathrm{seq}}$ samples, respectively, satisfying $m_{\mathrm{g}} < m_{\mathrm{seq}}$.

\end{theorem}
The theorem demonstrates that the graph-based method has a lower Bayes risk and could achieve superior performance compared to the ESB method under identical conditions. 
The detailed proof is provided in Appendix~\ref{sec:proof}.

\begin{table*}[t]
\caption{Action error detection results(\%), the metric is the accuracy($\uparrow$). GPT4o. is short for GPT4o-mini. The best performances are in \textcolor{red}{\textbf{bold}}, and the second-best method is \textcolor{blue}{\underline{underlined}}.}
\label{tab:detection}
\centering
\resizebox{\textwidth}{!}{\begin{tabular}{l@{\hspace{0.20cm}}c@{\hspace{0.2cm}}c@{\hspace{0.10cm}}c@{\hspace{0.2cm}}c@{\hspace{0.2cm}}c@{\hspace{0.1cm}}c@{\hspace{0.2cm}}c@{\hspace{0.2cm}}c@{\hspace{0.1cm}}c@{\hspace{0.2cm}}c@{\hspace{0.2cm}}c@{\hspace{0.2cm}}c@{\hspace{0.2cm}}c@{\hspace{0.2cm}}} 
    \Xhline{1.2pt}
    \textbf{}  & \multicolumn{3}{c}{\textbf{AlfWorld}} & \multicolumn{3}{c}{\textbf{TextWorld}} & \multicolumn{3}{c}{\textbf{ScienceWorld}} & \multicolumn{3}{c}{\textbf{TravelPlanner}}\\
    \cmidrule(lr){2-4} \cmidrule(lr){5-7} \cmidrule(lr){8-10} \cmidrule(lr){11-13}
    Method  & GPT4o. & Qwen2.5 & Gemma3 & GPT4o. & Qwen2.5 & Gemma3 & GPT4o. & Qwen2.5 & Gemma3 & GPT4o. & Qwen2.5 & Gemma3 \\
    \Xhline{1pt}
    \rowcolor{gray!30} \multicolumn{13}{c}{\textbf{Text Classification}} \\
    TF-IDF & \textcolor{blue}{\underline{58.04}}	&	\textcolor{blue}{\underline{66.98}}	&	\textcolor{blue}{\underline{63.75}}		&	63.68	&	49.82	&	42.09		&	81.02	&	\textcolor{blue}{\underline{80.66}}	&	\textcolor{blue}{\underline{80.52}}		&	\textcolor{red}{\textbf{93.35}}	&	65.23	&	\textcolor{blue}{\underline{65.22}} \\
    Bert   & 53.09	&	62.07	&	55.85		&	\textcolor{blue}{\underline{64.17}}	&	\textcolor{blue}{\underline{54.12}}	&	49.44		&	\textcolor{blue}{\underline{83.51}}	&	80.48	&	78.72		&  92.14	&	\textcolor{blue}{\underline{65.76}}	&	64.63 \\ 
    \rowcolor{gray!30} \multicolumn{13}{c}{\textbf{Retrieve}}\\
    MiniLM & 39.33	&	46.74	&	49.74		&	59.37	&	43.55	&	42.37		&	73.87	&	76.57	&	73.76		&	90.20	&	60.56	&	58.79 \\
    E5     & 44.82	&	57.51	&	53.14		&	57.05	&	27.24	&	52.54		&	76.59	&	74.48	&	72.98		&	90.88	&	58.22	&	59.23 \\
    GTR    & 48.10	&	58.41	&	54.00		&	60.20	&	30.47	&	53.67		&	76.17	&	74.74	&	74.00		&	90.09	&	59.77	&	59.30 \\
    \rowcolor{gray!30} \multicolumn{13}{c}{\textbf{RAG}} \\
    GPT4o  & 55.18	&	62.73	&	60.38		&	52.74	&	49.10	&	\textcolor{red}{\textbf{69.21}}		&	72.15	&	74.21	&	68.76		&	79.87	&	56.82	&	55.72 \\
    Gemma3 & 55.74	&	64.43	&	59.95		&	61.69	&	51.25	&	60.17		&	77.11	&	80.75	&	76.59		&	34.95	&	39.81	&	39.28 \\
    Qwen2.5 & 50.04	&	57.56	&	55.32		&	43.12	&	35.84	&	51.69		&	67.94	&	71.05	&	66.30		&	24.25	&	31.22	&	31.29 \\
    \rowcolor{gray!30} \multicolumn{13}{c}{\textbf{LLM Zero-Shot}} \\
    GPT4o  & 32.12	&	30.13	&	32.06		&	45.77	&	42.29	&	48.31		&	33.33	&	35.30	&	31.45		&	4.21	&	8.68	&	8.03 \\
    Gemma3 & 31.04	&	28.92	&	33.81		&	48.26	&	34.59	&	44.35		&	18.81	&	34.42	&	20.30		&	2.42	&	4.48	&	4.03 \\
    Qwen2.5& 16.46	&	20.24	&	16.75		&	39.64	&	27.78	&	31.64		&	19.76	&	19.79	&	20.83		&	2.25	&	10.86	&	11.46 \\
    \rowcolor{gray!30} \multicolumn{13}{c}{\textbf{LLM One-Shot}} \\
    GPT4o  & 22.34	&	21.81	&	28.22		&	51.41	&	37.46	&	55.65		&	13.49	&	13.43	&	11.03		&	4.32	&	8.89	&	7.96 \\
    Gemma3 & 27.30	&	20.99	&	32.39		&	42.45	&	32.62	&	37.85		&	9.78	&	21.52	&	14.31		&	2.55	&	4.46	&	4.19 \\
    Qwen2.5& 17.12	&	17.59	&	17.05		&	40.30	&	33.15	&	38.98		&	31.12	&	17.96	&	21.44		&	2.58	&	11.19	&	11.55 \\
    \rowcolor{gray!30} \multicolumn{13}{c}{\textbf{LLM Three-Shot}} \\
    GPT4o  & 29.13	&	20.65	&	33.51		&	50.41	&	35.13	&	49.72		&	15.85	&	15.03	&	11.77		&	3.46	&	2.48	&	4.42 \\
    Gemma3 & 30.21	&	26.29	&	27.26		&	38.47	&	34.77	&	33.33		&	29.23	&	50.91	&	20.83		&	2.27	&	5.59	&	8.73 \\
    Qwen2.5& 16.33	&	19.29	&	16.59		&	44.28	&	26.70	&	37.85		&	28.54	&	31.97	&	22.63		&	0.95	&	5.50	&	5.57 \\
    \hline
\rowcolor{gray!10} 
    Ours  & \textcolor{red}{\textbf{63.98}}   &	\textcolor{red}{\textbf{69.89}}	&	\textcolor{red}{\textbf{66.85}}		&	\textcolor{red}{\textbf{67.00}}	&	\textcolor{red}{\textbf{62.19}}	&	\textcolor{blue}{\underline{66.67}}		&	\textcolor{red}{\textbf{87.84}}	&	\textcolor{red}{\textbf{84.57}}	&	\textcolor{red}{\textbf{83.39}}		& \textcolor{blue}{\underline{92.75}}	&	\textcolor{red}{\textbf{67.43}}	&	\textcolor{red}{\textbf{68.15}} \\
    \Xhline{1.2pt}
\end{tabular}
}
\end{table*}

\section{Experiments}
To verify the effectiveness of our framework, we conduct empirical studies on two perspectives: error detection and feedback evaluation. We first collect and build datasets on four environments, and then we compare the performance on our framework and baselines.

\textbf{Datasets \& Benchmarks.} Four environments are used for dataset construction, including {\alfw}~\citep{shridhar2020alfworld}, {\textw}~\citep{cote18textworld,jansen2022textworldexpress}, {\sciw}~\citep{wang2022scienceworld}, and {\trapl}~\citep{xie2024travelplanner}. These benchmarks evaluate long-term reasoning by requiring agents to interact with the environment, gather information, and make decisions. For instance, ScienceWorld evaluates an agent’s ability to programmatically solve problems using scientific knowledge, while {\trapl} tests its capacity to use tools for information gathering and planning under user constraints.
We follow the React~\citep{yao2023react} to implement a vanilla agent for the data collection and experiments. In AlfWorld, we follow Agentboard to split the train, validation, and test tasks. In TextWorld and ScienceWorld, we use the first 10 variants as training and validation tasks, and 5 variants as test tasks.  In TravelPlanner, we select 100 tasks for each difficult level as test tasks and reserve the others for training and validation tasks. To remove the bias of LLMs, we use GPT4o-mini~\citep{openai_gpt4o_mini}, Qwen2.5-14B~(Qwen2.5)~\citep{yang2025qwen3}, and Gemma3-27B~(Gemma3)~\citep{team2025gemma} for all four environments.
For the step-level annotation, we first employ LLM models to generate the label, and then a rule-based script is used to select and filter. The detailed dataset information is available in Appendix~\ref{sec:dataset}.

\textbf{Evaluation Metrics.}
For error detection, we use the classical metric, accuracy, to measure the performance. In the feedback evaluation, we report two metrics: Pass Ratio~(PR) and Ground Ratio~(GR).  PR is used to evaluate whether an agent has successfully completed a given task, and GR is a metric that assesses the validity of an agent's action within a given environment state, serving as an indicator of its grounding and understanding. In TravelPlanner, the PR indicates whether the agents give the final plan. In practice, we are primarily concerned with the PR. We mainly provide the analysis of GR in the Appendix.

\textbf{Baselines.}
Our baselines contain three categories of approaches. The conventional methods include TF-IDF~\citep{salton1988term} and fine-tuned Bert methods~\citep{devlin2019bert}, which are represent the text classification approaches. We use TF-IDF to extract the feature and logistic regression to predict the error probability.  
We also use retrieval-based~\citep{guu2020retrieval} and Retrieval-Augmented Generation~(RAG)~\citep{lewis2020retrieval} methods as representative approaches for incorporating external databases. In retrieval-based methods, we use embedding to find the most similar trajectory. We use three language models to obtain the embeddings, including ALL-MiniLM-L6-v2~(MiniLM)~\citep{minilmv2}, E5-Large~(E5)~\citep{wang2022text}, and GTR-T5-Large~(GTR)~\citep{ni2022large}. For RAG, we select the five most similar trajectories as candidates and use LLMs to generate the answer. We use the GTR as the embedding model to measure the similarity.
Furthermore, we consider the LLM-as-a-judge~\citep{zheng2023judging} methods as the LLM-based methods, including three settings: zero-shot, one-shot, and three-shot. To avoid the model bias, we use three LLM models for RAG and LLM-as-judge: GPT4o~\citep{openai2024gpt4ocard}, Qwen2.5-14B~(Qwen2.5), and Gemma3-27B~(Gemma3).

\subsection{Error Probe Performance Analysis}
\label{subsec:dectection}
\stitle{Detection Results.} In Table~\ref{tab:detection}, we report the results across four benchmark datasets. Overall, our method mostly outperforms baselines. Specifically, the graph-based detection approach achieves over a $5\%$ improvement compared to text classification methods on average. 
The advantage is most striking in the AlfWorld environment, where the agent is built with GPT4o-mini, and our method achieves a $10.23\%$ increase over the second-best approach. 
This observation aligns with our analysis that graph-based methods require lower data complexity than sequence-based approaches, making them more effective in handling sparse or noisy trajectories.

We also observe interesting differences across RAG methods. In the first three environments, RAG methods outperform standard retrieval approaches, whereas in TravelPlanner, the opposite holds. We attribute this reversal to the much longer observation sequences in TravelPlanner, which may dilute the benefits of RAG and make simple retrieval more effective. Meanwhile, LLM-as-judge methods consistently perform poorly. We believe this is due to the inherent difficulty of the tasks, which demand strong logical reasoning skills beyond the capabilities of current judgment-style approaches. Finally,
\begin{wraptable}{r}{0.50\textwidth}
\caption{{Statistical significance analysis of action error detection~(Micro F1) on TextWorld.The best performances are in \textcolor{red}{\textbf{bold}}, and the second-best method is \textcolor{blue}{\underline{underlined}}.} }
\centering
\scalebox{0.85}{ \begin{tabular}{lccc} 
    \Xhline{1.2pt}
    {Method}  & {GPT4o.} & {Qwen2.5} & {Gemma3} \\ 
    \Xhline{1pt}
    \rowcolor{gray!30} \multicolumn{4}{c}{{\textbf{Text Classification}}} \\
    {TF-IDF} & \ms{{0.6365}}{{0.0016}} & \ms{\textcolor{blue}{\underline{0.4996}}}{{0.0041}} & \ms{{0.4045}}{{0.0160}}  \\
    {Bert}   & \ms{\textcolor{blue}{\underline{0.6368}}}{{0.0097}} & \ms{\textcolor{blue}{\underline{0.4996}}}{{0.0501}} & \ms{{0.4345}}{{0.0493}}  \\
    \rowcolor{gray!30} \multicolumn{4}{c}{{\textbf{Retrieval}}}\\
    {GTR}    &  \ms{{0.5322}}{{0.0073}} & \ms{{0.3018}}{{0.0009}} & \ms{{0.5322}}{{0.0073}}  \\
    \rowcolor{gray!30} \multicolumn{4}{c}{{\textbf{RAG}}} \\
    {Gemma3} &  \ms{{0.6232}}{{0.0030}} & \ms{{0.5183}}{{0.0066}} & \ms{\textcolor{blue}{\underline{0.6124}}}{{0.0072}} \\
    % \hline
    \rowcolor{gray!30} \multicolumn{4}{c}{{\textbf{LLM Zero-Shot}}} \\
    {Gemma3} & \ms{{0.4732}}{{0.0030}} & \ms{{0.3427}}{{0.0029}} & \ms{{0.4492}}{{0.0000}} \\
    \rowcolor{gray!30} \multicolumn{4}{c}{{\textbf{LLM One-Shot}}} \\
    {Gemma3} & \ms{{0.3755}}{{0.0755}} & \ms{{0.2620}}{{0.0741}} & \ms{{0.3090}}{{0.0554}} \\
    \rowcolor{gray!30} \multicolumn{4}{c}{{\textbf{LLM Three-Shot}}} \\
    {Gemma3} & \ms{{0.1158}}{{0.0457}} & \ms{{0.3194}}{{0.0505}} & \ms{{0.1876}}{{0.0907}} \\
    \hline
    \rowcolor{gray!30} {Ours} & \ms{\textcolor{red}{\textbf{0.6872}}}{{0.0091}} & \ms{\textcolor{red}{\textbf{0.5810}}}{{0.0133}} & \ms{\textcolor{red}{\textbf{0.6316}}}{{0.0229}}  \\
    \Xhline{1.2pt}
\end{tabular}
}
\label{tab:det_robustness}
\vspace{-0.4cm}
\end{wraptable}
we find no consistent trend between one-shot and three-shot settings. Surprisingly, in most cases, the zero-shot setting achieves better performance than either, likely because additional samples introduce bias that hinders the reasoning ability of the LLMs. These findings further highlight the robustness of our graph-based detection approach across different environments and prompting strategies.

\stitle{Statistical Significance Analysis.} To further evaluate the robustness of error detection performance, we conduct the statistical significance analysis on the TextWorld benchmark with three LLMs. We use Gemma3 for the LLM zero-shot, one-shot, and three-shot methods. As shown in Table~\ref{tab:det_robustness}, our method consistently outperforms the baselines, demonstrating the advantage of the action-centric graph. We further provide macro F1 results in Appendix~\ref{subsec:error_detect}.

\subsection{Feedback Evaluation Results}
\label{subsec:feedback}
\stitle{Main Results.} To evaluate the effectiveness of {\pipeline}, we conduct feedback evaluation experiments. We compare against four baselines: TF-IDF, GTR, Gemma3, and GPT4o, representing text-classification, retrieval-based, RAG-based, and LLM-as-judge (zero-shot) categories, respectively. The detailed experimental setup is provided in Appendix~\ref{subsec:pipeline}. 
We report the average results on four benchmarks and three LLM agents in Figure~\ref{fig:feedback_results}. As the results show,
both our method and baselines outperform the vanilla results in most cases, demonstrating that the mechanism, 
\begin{wraptable}{r}{0.50\textwidth}
\centering
\caption{Performance comparison~(Pass Ratio(\%)) of different methods with two open-source LLMs.The best performances are in \textcolor{red}{\textbf{bold}}.}
\label{tab:eval_results}
\begin{tabular}{lcccc}
\toprule
 & {ReAct} & {Reflexion} & {Ours} \\
\midrule
{Qwen2.5} & 62.50  & 60.00 & \textcolor{red}{\textbf{65.00}} \\
{Gemma3}  & 70.00 & 67.50  & \textcolor{red}{\textbf{75.00}} \\ 
\bottomrule
\end{tabular}
\end{wraptable}
feedback in the interaction consistently enhances agent performance across different detection strategies. Among these methods, {\pipeline} achieves the strongest performance on average, which outperforms baselines on all four benchmarks. On average, our method improves the vanilla by $14.69\%$. 
Interestingly, even for RAG and LLM-as-judge, which show weaker detection accuracy, the agents still benefit from performance gains. For example, in ScienceWorld, LLM-as-Judge methods have a poor detection ratio but still boost the agents. A possible explanation is that LLMs may not always follow instructions precisely, leading to incorrect prediction labels; however, they still provide useful analysis that guides the agents. 
For the detailed results and analysis of GR, we provide in Appendix~\ref{subsec:feedback_detail_results}.
\begin{figure*}[t]
\centering
\includegraphics[width=\textwidth]{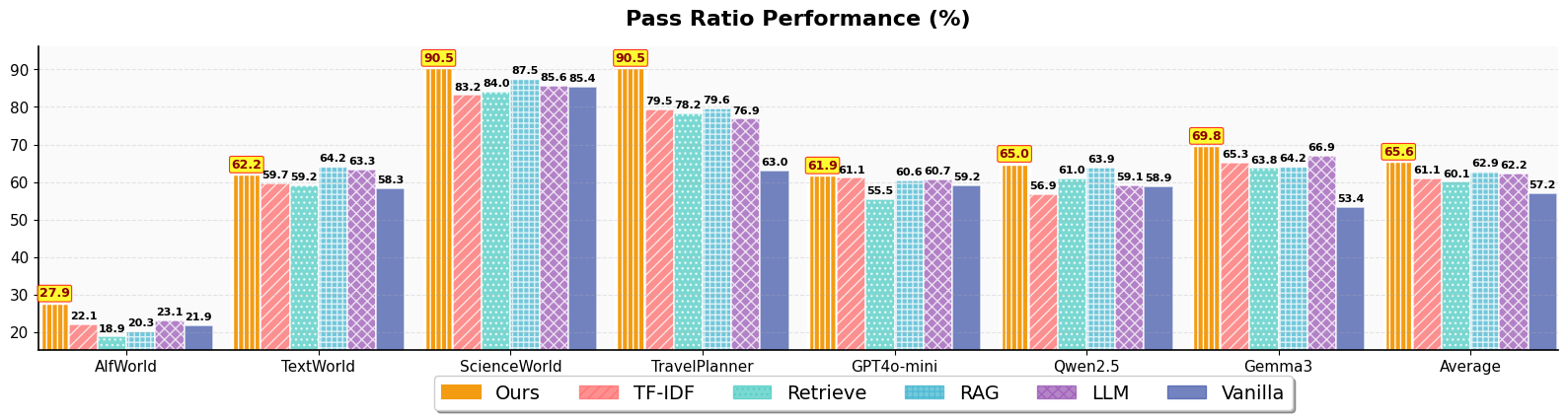}
\vspace{-0.6cm}
\caption{The pass ratio~(\%) results across four benchmarks and three LLM agents. Ours consistently outperforms other baselines.}
\label{fig:feedback_results}
\vspace{-0.6cm}
\end{figure*}

\stitle{Comparison with Deliberate Prompting-based LLM agents.} To further illustrate the effectiveness of ours, we conduct the comparison experiments with ReAct~\cite{yao2023react} and Reflexion~\cite{reflexion}, which are famous prompting-based LLM agents. We use the TextWorld benchmark with two open-source LLMs. 
\begin{wraptable}{r}{0.65\textwidth}
\vspace{-0.3cm}
\caption{Graph construction ablation study on detection. The Dir. and Undir. are short for directed graph and undirected graph.}
\label{tab:detect_ablation}
\centering
    \scalebox{0.8}{
        \begin{tabular}{l@{\hspace{0.2cm}}c@{\hspace{0.2cm}}c@{\hspace{0.2cm}}c@{\hspace{0.1cm}}c@{\hspace{0.2cm}}c@{\hspace{0.2cm}}c@{\hspace{0.1cm}}c} % {\linewidth}
        \toprule
        \textbf{}  & \multicolumn{3}{c}{\textbf{AlfWorld}} & \multicolumn{3}{c}{\textbf{ScienceWorld}} \\
        \cmidrule(lr){2-4} \cmidrule(lr){5-7} 
        Graph & GPT4o. & Qwen2.5 & Gemma3 & GPT4o. & Qwen2.5 & Gemma3 & Avg. \\
        \midrule
        \rowcolor{gray!30} \multicolumn{8}{c}{\textbf{Onehot}} \\
        Dir.   & \textcolor{red}{\textbf{64.22}}	& 69.38	& 65.60 & 85.95	& 81.99	& 81.02 & 74.69\\
        Undir. & 62.44	& 67.83	& 65.63 & 82.21	& 82.17	& 81.71 & 73.67\\
        \rowcolor{gray!30} \multicolumn{8}{c}{\textbf{Bert}} \\
         Dir.     & \textcolor{blue}{\underline{63.98}}	& \textcolor{blue}{\underline{69.89}}	 & \textcolor{blue}{\underline{66.85}} & \textcolor{red}{\textbf{87.84}}	&\textcolor{blue}{\underline{84.57}}	&\textcolor{blue}{\underline{83.39}}  & \textcolor{red}{\textbf{76.09}}\\
         Undir.   & 63.05	& \textcolor{red}{\textbf{70.18}}	 & \textcolor{red}{\textbf{67.22}} & 84.32	&\textcolor{red}{\textbf{85.68}}	&\textcolor{red}{\textbf{83.72}} & \textcolor{blue}{\underline{75.70}}\\
        \bottomrule
        \end{tabular}
    }
\end{wraptable}
As shown in Table~\ref{tab:eval_results}, our method outperforms both baselines. This is primarily because ReAct and Reflexion rely heavily on the inherent reasoning capabilities of the LLM, which can lead to error accumulation and degraded performance, especially when the underlying model is not sufficiently large, and no external feedback signals are available.

\subsection{Ablation Study}
\label{sec:ablation}
\stitle{Graph Construction.} To examine the graph's variance, we conducted ablation studies on another three settings, exploring combinations of two node features, Bert embedding or one-hot embedding as attribution,  and two edge types, directed and undirected edges. As shown in Table \ref{tab:detect_ablation}, our results indicate that the performance of undirected graphs is slightly inferior to that of directed graphs.  The key reason is that POMDP dependencies are inherently directional,  and representing them with undirected graphs introduces noise, leading to a slight performance decline. Additionally, using one-hot encoding as a feature yields poorer performance compared to using Bert, demonstrating that text information is useful in the detection tasks.

\stitle{Feedback Integration.} To explore the insight of the integration mechanisms, we evaluate four types of integration mechanisms: (1) incorporating feedback in the user prompt, (2) including feedback in the user prompt without label explanations, (3) adding feedback in the system prompt, and (4) appending feedback to the trajectory history. We conduct these experiments in the TextWorld environment with two open-source LLMs. As the results shown in Table~\ref{tab:ablation_integration}, Gemma3 is more robust than Qwen2.5 across different feedback integration strategies. When feedback is incorporated into the user prompt, it yields greater improvements than other configurations, even after integration with the system prompt.

\begin{figure}[t]
  \centering
  \begin{subfigure}{0.45\linewidth}
    \includegraphics[width=\textwidth]{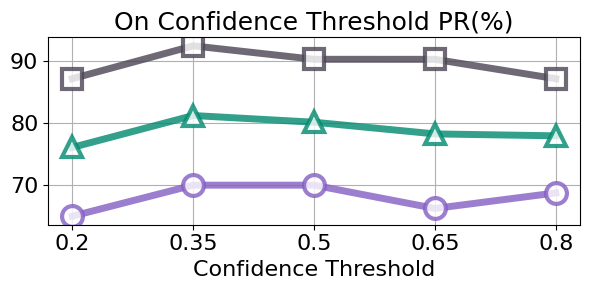}   
    \label{fig:ablation_threshold}
  \end{subfigure}
  \begin{subfigure}{0.45\linewidth}
    \includegraphics[width=\textwidth]{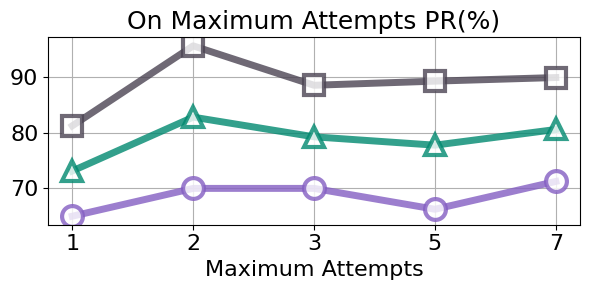}   
        \label{fig:ablation_attempts}
  \end{subfigure}
  \vspace{-0.5cm}
  \caption{The ablation study of pass ratio~(\%) on maximum attempt times and confidence threshold.}
  \vspace{-0.5cm}
  \label{fig:ablation}
\vspace{-0.2cm}
\end{figure}

\begin{wraptable}{r}{0.50\textwidth}
\caption{{Ablation study results~(PR) of feedback integration mechanisms on TextWorld. In the table, ``exp.'' is short for ``explanations''. The best performances are in \textcolor{red}{\textbf{bold}}.} }
\centering
\scalebox{1.0}{ \begin{tabular}{lcc} 
    \Xhline{1.2pt}
    {Method}  & {Qwen2.5} & {Gemma3} \\ 
    \Xhline{1pt}
    {in trajectories}  & {57.50}  & {70.00}    \\
    {in system}        & {60.00}  & \textcolor{red}{\textbf{75.00}}  \\
    {in user}          & \textcolor{red}{\textbf{65.00}}  & \textcolor{red}{\textbf{75.00}}   \\
    {in user w/o exp.} & {62.50}  & {72.50}  \\
    \Xhline{1.2pt}
\end{tabular}
}
\label{tab:ablation_integration}
\end{wraptable}
\stitle{Hyperparameter Study of Feedback.} Furthermore, we conduct ablation studies to analyze the impact of different settings of {\pipeline} on the TextWorld and ScienceWorld benchmarks using Qwen2.5 and Gemma3. Specifically, we analyze two key factors: the maximum number of attempts and the choice of confidence threshold, as detailed in Appendix~\ref{subsec:pipeline}. As illustrated in Figure~\ref{fig:ablation}, PR performance improves as the number of attempts increases, then decreases, before eventually stabilizing. 
This behavior is intuitive; additional attempts raise the likelihood of producing valid and reasonable actions, thereby reducing the probability of failure at each step. However, we also observe a decline in GR, consistent with the findings in Section~\ref{subsec:feedback}.
In contrast, the confidence threshold has only a marginal effect on overall performance. We attribute this to the strength of the detection module, where the threshold plays a minor role, when detection is highly accurate and robust, as relatively few false alarms or misclassifications propagate to the next stage. If detection were less reliable, the choice of threshold would likely have a much greater impact. Interestingly, we also find that setting a higher confidence threshold slightly improves GR performance.
Additional quantitative results supporting these observations are provided in Appendix~\ref{subsec:ablation_detail_results}.

\section{Conclusion}
In this paper, we propose {\ours}, a novel PGM-based graph detection method for step-level diagnosis of agent action decisions. Compared with traditional approaches and LLM-based methods such as text classification, RAG, or LLM-as-judge, {\ours} achieves lower error rates while requiring fewer samples. Beyond the detection module itself, we further introduce {\pipeline}, a flexible pipeline that integrates the detection module as an independent sandbox to provide actionable feedback on agent behaviors.
We conduct extensive experiments on four benchmarks and three LLM-based agents to validate the effectiveness of our approach. The experiments and findings highlight not only the effectiveness of {\ours} in detecting errors but also its potential as a general framework for improving decision-making in LLM agents. We believe this work can inspire more reliable performance-enhancement strategies for agent-based systems.

\stitle{Limitation.} This work requires collecting a dataset, which can be time-consuming for tasks with limited historical trajectory data. Although continuously updating the trajectory pool and action-centric graph could further improve agent performance, we do not explore this aspect in the current paper due to space constraints. We leave it as a direction for future work.

\stitle{Broader Impact.} This work introduces a method for improving the performance of existing LLM agents, formalized as a plug-in module that integrates seamlessly with existing agent frameworks. We believe this work contributes positively to the broader goal of building more capable and reliable AI systems, with potential benefits across a range of real-world applications.

\bibliographystyle{plain}
\bibliography{example_paper}

\appendix
\newpage
\section{Related Work}
\stitle{LLM Agents.} Empowered by LLMs, agents have experienced rapid growth and demonstrated remarkable performance across a wide range of tasks, including goal reasoning and action execution~\citep{xi2025rise}.For instance, LLMs have empowered embodied agents~\citep{chen2023robogpt} with perception, interaction, and planning skills for versatile operation in virtual and physical environments~\citep{londono2024fairness}. To address the long-horizon interaction tasks, existing methods can be divided into two categories: fine-tuning-based and fine-tuning-free methods. To obtain a refined language model agent, fine-tuning-based methods~\citep{wang2025steca,xiong2024watch,wang2024e2cl,song2024trial,wang2023math} enhance decision-making capabilities by tuning LLMs from expert demonstrations or exploration~\citep{chen2023fireact,yin2023agent,xiang2023language,song2024trial}. Another line of work incroporates external tools/models to gain improvements, such as structure search~\citep{yao2023tree,besta2024graph,hao2023reasoning,zhuang2023toolchain} and retrieval~\citep{xiao2023o3d,kagaya2024rap,zhou2024trad}. These methods typically guide LLM agents by incorporating external knowledge. For instance, structure search offers feedback from the environment~\citep{xiang2023language}, while retrieval selects optimal actions by comparing them to offline successful trajectories~\citep{kagaya2024rap}.

\stitle{Graph in LLM Reasoning.} To enhance the reasoning capacity of LLM, recent works~\citep{yao2023tree,besta2024graph,wang2022self}, such as chain-of-thoughts~\citep{wei2022chain}, are proposed. These methods essentially decompose the LLM’s reasoning into nodes and edges, where nodes represent entities and edges represent thought processes~\citep{besta2024graph}, thereby modeling interactive relationships. Leveraging these relationships can enhance the LLM’s reasoning capabilities. Alternatively, explicit graph structures like knowledge graphs~\citep{mavromatis2025gnn} can serve as external knowledge bases to guide reasoning, such as GraphRAG~\citep{peng2024graph,luo2024graph,li2024decoding}, Knowledge Graph Question Answering~(KGQA)~\citep{lan2022complex,ye2021rng,zhang2022arl}. In ~\citep{wang2024understanding}, the authors propose that the reasoning ability of the language model can be seen as an aggregation of the numerous indirect reasoning paths encountered during pretraining. Moreover, integrating graph structures with LLMs can enhance their reasoning abilities. For instance, ~\citep{li2025injecting} embeds knowledge graph representations directly into LLM tokens as graph semantics, enabling the model to incorporate structural information without relying on prompt engineering or extensive fine-tuning.  {Furthermore, graphs are increasingly used to ground LLM reasoning in actionable plans and environments. For instance, graphs can structure high-level subgoals~\citep{lee2022dhrl}, represent physical scenes for failure recovery~\citep{yu2025scene}, or facilitate graph-based learning to improve the planning robustness of LLM agents~\citep{wu2024can}.}

\section{Proofs of Theorem \ref{th:1}}
\label{sec:proof}
\begin{proof}
We prove the proof of Theorem \ref{th:1} in several steps.
\paragraph{(1) Sufficiency.} Assumption~\ref{assump:A1} states exactly that \(S\) d-separates \(Y\) from the rest of the observed variables, i.e. \(Y\perp\!\!\!\perp\ X\setminus S\mid S\). This is equivalent to the conditional probability factorization $\mathbb{P}(Y\mid X)=\mathbb{P}(Y\mid S)$. Hence \(S\) is a sufficient statistic for \(Y\) relative to \(X\).

\paragraph{(2) Bayes risk equality.} Given the full input \(X\), the Bayes classifier and its risk are:
\[
h^*_X(x) = \arg\max_{c\in\{1,\dots,C\}} \mathbb{P}(Y=c\mid X=x), \text{and} ~ R^*(X) = \mathbb{E}\big[ 1\{h^*_X(X)\neq Y\}\big].
\]
Because \(\mathbb{P}(Y\mid X)=\mathbb{P}(Y\mid S)\), the Bayes posterior and classifier can be written using \(S\) only:
\[
h^*_X(x) = \arg\max_c \mathbb{P}(Y=c\mid S=\Phi_{\mathrm{g}}(x)) \equiv h^*_S(\Phi_{\mathrm{g}}(x)).
\]
Therefore \(R^*(X) = R^*(S)\): no additional reduction in Bayes risk is possible by observing \(X\) instead of \(S\).

\paragraph{(3) Information-theoretic ordering.} For any representation \(U=\Phi(X)\), consider the Markov chain \(Y\to X\to U\). By the data processing inequality, $I(Y;U) \le I(Y;X)$. Because \(S\) is a (deterministic) function of \(X\) and suffices for \(Y\), we also have \(I(Y;S) = I(Y;X)\). Hence
\( I(Y;U) \le I(Y;S)\).

A standard information-theoretic lower bound relates classification error to the remaining uncertainty \(H(Y)-I(Y;U)\). Therefore a representation with strictly larger mutual information with \(Y\) yields a strictly smaller lower bound on achievable error; hence if \(I(Y;S)>I(Y;U)\) then \(R^*(S) < R^*(U)\).

\paragraph{(4) Information-Theoretic Sample Complexity.} We now derive the lower bound on the number of samples $m$ required to achieve an expected classification error $\epsilon$. Let the true class label $Y \in \{1, \dots, C\}$ be uniformly distributed, such that the initial entropy is $H(Y) = \log C$. Let $\hat{Y}$ be the predicted label derived from a dataset of $m$ independent samples $S_m = \{X_1, \dots, X_m\}$.

By Fano's Inequality, the conditional entropy of the true label given the prediction is bounded by the probability of error $\epsilon = \mathbb{P}(\hat{Y} \neq Y)$:
$$H(Y \mid \hat{Y}) \le H(\epsilon) + \epsilon \log(C - 1)$$
To simplify the bound, we upper-bound the binary entropy $H(\epsilon) \le 1$ (assuming logarithms are base-2) and $\log(C - 1) < \log C$, yielding:
$$H(Y \mid \hat{Y}) < 1 + \epsilon \log C$$

Since the prediction $\hat{Y}$ is a function of the data $S_m$, they form a Markov chain $Y \to S_m \to \hat{Y}$. By the Data Processing Inequality, $H(Y \mid S_m) \le H(Y \mid \hat{Y})$. Furthermore, by the definition of mutual information, $H(Y \mid S_m) = H(Y) - I(Y ; S_m) = \log C - I(Y ; S_m)$. Substituting these into the relaxed Fano bound gives:
$$\log C - I(Y ; S_m) < 1 + \epsilon \log C$$
Rearranging to isolate the mutual information:
$$I(Y ; S_m) > (1 - \epsilon) \log C - 1$$

Assuming the $m$ samples are drawn independently, the total information gained about $Y$ from $S_m$ is bounded by $m$ times the information gained from a single representation $X$, such that $I(Y ; S_m) \le m I(Y ; X)$. Substituting this upper bound yields:
$$m \cdot I(Y ; X) > (1 - \epsilon) \log C - 1$$
$$m \gtrsim \frac{(1 - \epsilon)\log C - 1}{I(Y ; X)}$$

Having established this general lower bound, we substitute the representations from Section~\ref{sec:theory} to obtain the sample complexity for the respective models:
$$m_{\mathrm{g}} \propto \frac{(1-\epsilon)\log C - 1}{I\!\left(Y ; S\right)}, \quad m_{\mathrm{seq}} \propto \frac{(1-\epsilon)\log C - 1}{I\!\left(Y ; U\right)}$$

From paragraph (3), we established that $I(Y ; S) \ge I(Y ; U)$. Because the mutual information term resides in the denominator, a larger information capacity requires fewer samples to achieve the same error threshold $\epsilon$. It directly follows that $m_{\mathrm{g}} \le m_{\mathrm{seq}}$, with strict inequality if $I(Y ; S) < I(Y ; U)$. This completes the proof.

\end{proof}

\section{Detailed Implementation}
\subsection{Error Definition}
\label{sec:label_definition}
As described in Section~\ref{sec:problem}, we utilize the milestone state definition to detect the errors, which is based on the previous actions and goals.  Besides, we also consider the repeated action as an error. In this work, we define the error using the following categories:
\begin{itemize}[itemsep=1.5pt,topsep=0pt,parsep=0pt,leftmargin=*]
    \item {No Error~(N.E.)}: The current action goes directly towards the next milestone state.
    \item {Illegal Action~(I.A.)}: The current action is not a valid action for the current environment.
    \item {Repeated Action~(R.A.)}: The current action has been done in the trajectory, and the results are the same.
    \item {Incorrect Target~(I.T.)}: The current action causes the agent to grab a wrong object or move to a wrong destination.
    \item {Precondition Not Met~(P.M.)}: The current action is valid, but it can only be executed when the agent has finished a specific action.
    \item {Condition Met, Action Not Taken~(C.M.)}: The current action is a valid action, but not necessary for the next milestone state or goal.
\end{itemize}
Notably, these error definitions are not specific to any particular environment. As a result, they do not include environment-related errors, such as tool misuse, which makes them easy to extend to other environments. 

{To further illustrate the error space, we provide an example. Given a task such as ``clean a plate and put it on the countertop'', the key states include: (1) a plate is found, (2) the plate is cleaned, and (3) the plate is placed on the countertop.
Suppose the action-only trajectory is ``go to desk, go to basin, pick up plate 1, clean plate 1, go to desk 1.'' In this trajectory, the label for ``go to desk 1'' is ``Condition Met, Action Not Taken.''
Making this prediction only requires tracing back to the action ``clean plate 1'', because after that action, the plate becomes clean. This example demonstrates how the error space is defined and why the Markov blanket assumption is reasonable.}
\begin{table*}[t]
\caption{Action error statistic information of datasets.}
\label{tab:dataset_error}
\centering
  \resizebox{\textwidth}{!}{
\begin{tabular}{lcccccccccccc} 
    \Xhline{1.2pt}
    \textbf{}  & \multicolumn{6}{c}{\textbf{Train and Validate}} & \multicolumn{6}{c}{\textbf{Test}} \\
    \cmidrule(lr){2-7} \cmidrule(lr){8-13} 
    Agent  & N.E. & I.A. & I.T. & R.A. & P.M. & C.M. & N.E. & I.A. & I.T. & R.A. & P.M. & C.M. \\
    \Xhline{1pt}
    \rowcolor{gray!30} \multicolumn{13}{c}{\textbf{AlfWorld}} \\
    GPT4o-mini & 746 & 4410 & 410 & 51 & 403 & 2187 & 245 & 1546 & 126 & 14 &179 & 771 \\
    Qwen2.5 & 804 & 5612 & 544 & 552 & 2825 & 1105 & 268 & 1838 & 171 & 173 & 1082 & 351\\
    Gemma3 & 804 &2186 & 583 & 237 & 3334 & 1120 & 268 & 932 & 213 & 96 & 1150 & 367\\
    \hline
    \rowcolor{gray!30} \multicolumn{13}{c}{\textbf{TextWorld}}\\
    GPT4o-mini & 240 & 181 & 81 & 70 & 549 & 63 & 80 & 76 & 35 & 32 & 343 & 37\\
    Qwen2.5 & 240 & 298 & 67 & 88 & 214 & 60 & 80 & 218 & 46 & 42 & 142 & 30 \\
    Gemma3 & 240 & 146 & 57 & 50 & 244 & 49 & 80 & 26 & 43 & 23 & 151 & 31\\
    \hline
    \rowcolor{gray!30} \multicolumn{13}{c}{\textbf{ScienceWorld}} \\
    GPT4o-mini & 601 & 4943 & 57 & 72 & 369 &  186 & 288 & 2924 & 13 & 21 & 306 & 57\\
    Qwen2.5 & 600 & 4332 & 23 & 59 & 383 & 54 & 288 & 1778 & 12 & 11 & 149 & 11\\
    Gemma3 & 600 & 3399 & 37 & 263 & 65 & 60 & 288 & 1953 & 8 & 127 & 21 & 42\\
    \hline
    \rowcolor{gray!30} \multicolumn{13}{c}{\textbf{TravelPlanner}} \\
    GPT4o-mini & 12441 & 507 & 41 & 394 & 41 & 12 & 4210 & 177 & 24 & 127 & 3 & 0 \\
    Qwen2.5 & 9362 & 3731 & 143 & 502 & 37 & 18 & 2826 & 1181 & 57 & 220 & 14 & 10\\
    Gemma3 & 9341 & 3734 & 140 & 547 & 35 & 20  & 2821 & 1174 & 61 & 216 & 13 & 10 \\
    \Xhline{1.2pt}
    \end{tabular}}
\end{table*}

\begin{table}[h]
\caption{Action and trajectory statistic information of datasets. GPT4o. is short for GPT4o-mini.}
\label{tab:dataset_TA}
\centering
\renewcommand\tabcolsep{3pt}
\begin{tabular}{lcccccc} 
    \Xhline{1.2pt}
    \textbf{}  & \multicolumn{3}{c}{\textbf{(Sub) Trajectory}} & \multicolumn{3}{c}{\textbf{Action}} \\
    \cmidrule(lr){2-4} \cmidrule(lr){5-7}
    Agent  & train & validate & test & train & validate & test \\
    \Xhline{1pt}
    \rowcolor{gray!30} \multicolumn{7}{c}{\textbf{AlfWorld}} \\
    GPT4o. & 7386 & 821 & 2881 &  74127 & 8286 & 28704\\
    Qwen2.5 & 10297 & 1145 & 3883 & 63522 & 7288 & 25020 \\
    Gemma3  & 7437 & 827 & 3026 & 41432 & 4817 & 16352 \\
    \hline
    \rowcolor{gray!30} \multicolumn{7}{c}{\textbf{TextWorld}}\\
    GPT4o. & 1065 & 119 & 603 & 9760 & 1057 & 4706\\
    Qwen2.5 & 870 & 97 & 558 & 10286 & 1161 & 5923\\
    Gemma3 & 707 & 79 & 354 & 7131 & 721 & 4174 \\
    \hline
    \rowcolor{gray!30} \multicolumn{7}{c}{\textbf{ScienceWorld}} \\
    GPT4o. & 5605 & 623 & 3609 & 30415 & 3506 & 16497 \\
    Qwen2.5 &4905 & 546 & 2249 & 17557 & 1818 & 5277 \\
    Gemma3 & 3981 & 443 & 2439 & 17713 & 1953 & 7659\\
    \hline
    \rowcolor{gray!30} \multicolumn{7}{c}{\textbf{TravelPlanner}} \\
    GPT4o. & 12055 & 1340 &  4541 & 82691 & 9201 & 31017\\
    Qwen2.5 & 12413 & 1380 & 4308 & 72468 & 7996 & 23604 \\
    Gemma3 & 12435 & 1382 & 4295 & 72260 & 8079 & 23547\\
    \Xhline{1.2pt}
\end{tabular}
\end{table}

\subsection{Dataset Construction}
\label{sec:dataset}
To build the datasets, we use three LLMs for each environment. For AlfWorld, TextWorld, and ScienceWorld, we use GPT4o-mini~\citep{openai_gpt4o_mini}, Qwen2.5~\citep{yang2025qwen3}, and Gemma3~\citep{team2025gemma} to construct the Agent. { We first run the training tasks to collect agent trajectories and obtain the expert/golden trajectories from the environments. Each full trajectory will be divided into several subtrajectories by the sequence of actions.
Notably, the expert/golden trajectories indicate successful trajectories, which can be obtained by random sampling or provided by benchmarks. We then utilize an LLM to analyze the trajectory to generate initial labels. Finally, we use a rule-based script to filter and remap labels. } The Gemma3 was used for the initialization labels. For the detailed prompts, we provide them at the end of the Appendix. Then we filter and map the label manually.
For AlfWorld, we use the AgentBoard dataset, with 402 tasks for training and validation, and 134 tasks for testing. The TextWorld environment comprises eight subsets; for each, we use 10 variants for training/validation and 5 variants for testing. Similarly, ScienceWorld has 30 subsets, with the first 10 variants used for training/validation and the subsequent 5 variants for testing. For all three of these environments, expert trajectories serve as ``No Error'' samples. The TravelPlanner dataset contains three levels of tasks~(easy, normal, and hard), with 880 tasks for training/validation and 300 tasks for testing (100 for each level).
The detailed statistics information is available in Table~\ref{tab:dataset_TA} and ~\ref{tab:dataset_error}. {All data and code will be released once the paper is accepted.}

\subsection{Agent Implementation}
\label{sec:agent}
In this paper, we follow the AgentBoard~\citep{chang2024agentboard} to build interactive agents for AlfWorld, TextWorld, and ScienceWorld. Each task is capped at 50 steps. At every step, we supply the LLM with the full history of observations and actions to support decision-making, along with a short demonstration included in the prompt to improve performance. For AlfWorld and TextWorld, we also provide the list of available action options derived from the environment. In ScienceWorld, however, the action space is prohibitively large, so we instead provide an action template. For TravelPlanner, we adopt the two-stage methods from open-source implementation\footnote{\url{https://github.com/OSU-NLP-Group/TravelPlanner}}. The detailed prompts used in our experiments are provided at the end of the Appendix.

\subsection{{\pipeline} Implementation}
\label{subsec:pipeline}
In this paper, we implement {\pipeline} using in-context learning with carefully designed prompts. At each step, the agent is allowed up to three attempts to satisfy the criteria defined in {\ours}. If a step fails, we provide additional feedback to guide the regeneration of actions, including the error type, its definition, and representative failure examples. In our experiments, we include three failure examples by default to ensure sufficient guidance while maintaining prompt efficiency.
To reduce false positives, we apply a confidence threshold of 0.6 to filter out unreliable detection results. For the baseline methods, we adopt the same three-attempt strategy for a fair comparison. However, since methods such as retrieval-based approaches and LLM-as-a-judge do not produce explicit confidence scores, this threshold cannot be applied to them. Within this iterative feedback loop, we incorporate potential error information directly into the prompt to enable more informed decision-making and refinement.

The detailed prompt templates are provided at the end of the Appendix, with redundant components omitted where they overlap with the agent implementation described earlier. This design ensures both reproducibility and clarity while avoiding unnecessary repetition.

\subsection{{ Detection Implementation}}
In this paper, we employ a three-layer GCN~\citep{kipf2016semi} with edge attributions as the graph-based detector, with hidden dimensions of [768, 512, 6]. We use AdamW~\citep{loshchilov2018decoupled} as the optimizer, with a learning rate of 1e-3 and a weight decay of 5e-5. The TF-IDF baseline is implemented using scikit-learn~\footnote{\url{https://github.com/scikit-learn/scikit-learn}}
 with default hyperparameters.
For Bert, we fine-tune the model using AdamW with a learning rate of 5e-5. For the retrieval-based and RAG methods, we use the training set as the retrieval database to ensure a consistent and fair comparison across methods. Additionally, for both the RAG and LLM-as-a-Judge approaches, we provide the full prompt templates at the end of the Appendix.

Overall, we keep the training settings and hyperparameters consistent across different methods whenever applicable, in order to ensure a fair and controlled evaluation of performance differences.

\subsection{Test-Time Graph Injection and Inference}
\label{sec:injection}

During the testing phase or when applying {\ours} as a real-time diagnosis sandbox for LLM agents, the framework operates in an inductive setting to evaluate previously unseen trajectories. To ensure consistency between the training and inference phases, we dynamically integrate novel test trajectories into the pre-established action-centric graph without retraining the GNN detector. 

Let $G = (\mathcal{V}, \mathcal{E})$ denote the graph constructed from the training set, where the node set $\mathcal{V}$ encompasses the unique actions and the training trajectories. 

Given a novel test trajectory $\xi_{\text{test}} \notin \mathcal{T}$ generated by the agent, we first utilize NLTK and the pretrained Bert model to process and initialize the representations of its constituent actions and observations. To embed $\xi_{\text{test}}$, we perform a dynamic graph injection operation by introducing a new trajectory node $v_{\text{test}}$. The augmented node set is defined as $\mathcal{V}' = \mathcal{V} \cup \{v_{\text{test}}\}$.The topological relationships for $v_{\text{test}}$ are established by linking it to its corresponding action nodes $v_i \in \mathcal{A}$, following the exact same heuristic defined in Equation~\ref{eq:attr}. If the new trajectory contains previously unseen actions, they are initialized and added to $\mathcal{A}$. The augmented adjacency matrix $A'$ is formally updated as:
\begin{equation}
A'{ij} = \left\{ \begin{array}{cc}
1 & \text{if } A{ij} = 1 \text{ in the training graph} \\
1 & v_i \in \mathcal{A}, \text{and } v_j = v_{\text{test}}, \text{where } v_i \in \xi_{\text{test}} \\
0 & \text{otherwise}\end{array} \right.
\end{equation}
Once the augmented graph and $A'$ are constructed, we apply our trained GNN detector to perform $K$ rounds of message passing. For computational efficiency during step-level action debugging, this message passing can be localized to the $K$-hop neighborhood of $v_{\text{test}}$. The representation of the test trajectory, $h_{\text{test}}^{(K)}$, is aggregated inductively from its constituent actions. Finally, we apply the softmax function to obtain the error probability vector $Z_{\text{test}} = \text{softmax}(h_{\text{test}}^{(K)})$, which provides immediate feedback for the agent's subsequent decision-making.

\subsection{Compute Resources Used}
The experiments in this paper involve both open-source LLMs (Qwen2.5-14B, Gemma3-27B) and commercial LLMs (GPT-4o, GPT-4o-mini). For the commercial models, inference is performed via the Azure OpenAI Service API, called from standard Linux machines using CPU-only execution, with no dependency on local GPU availability.

All open-source LLMs are deployed locally and served on a single NVIDIA A100 GPU (80GB VRAM) at a time, ensuring consistent and reproducible runtime characteristics across model families. Each model is loaded and evaluated independently to isolate performance measurements and avoid cross-model interference. The full hardware platform consists of a Linux system equipped with eight NVIDIA A100 GPUs, running CUDA 12.8, Python 3.10, and PyTorch 2.6.0, providing ample capacity for iterative experimentation and parallel development workflows.

\section{Extensive Results} 
\label{sec:detail_results}

\subsection{Quality of Datasets}
To verify the quality of the dataset, we conduct experimental validation in this section. While human-based evaluation provides the most reliable form of quality assessment, the sheer volume of trajectories makes it impractical for large-scale applications. To address this limitation, we employ two different LLMs for automatic annotation and introduce a cross-evaluation strategy to improve label quality and robustness.

Specifically, we use a larger and independent foundation model, GPT-OSS-120B, to re-annotate the ScienceWorld trajectories that were originally generated by Qwen2.5 and Gemma3. In total, this process covers 11,762 actions annotated with Gemma3. We then compare the resulting class distribution with that produced by our primary annotation model. As shown in Table~\ref{tab:annotation_robustness}, the distribution of error labels remains highly consistent across different models. We further calculate Cohen's Kappa $\kappa$ as the agreement metric. For Qwen2.5 and Gemma3, we have the results $0.8817$ and $0.8907$. This consistency suggests that our annotation pipeline is robust and not overly sensitive to the biases or noise introduced by any single model. Furthermore, it provides empirical evidence that our dataset maintains stable labeling characteristics even under model variation, supporting its reliability for downstream evaluation and training.

\begin{table*}[t]
\caption{Quality and robustness analysis of dataset on ScienceWorld. The trajectories are generated by Qwen2.5 and Gemma3. The annotations are generated by GPT-OSS-120B(GPT-OSS) and Gemma3.}  
\centering
\scalebox{1.0}{ \begin{tabular}{lcccc} 
    \Xhline{1.2pt}
    & \multicolumn{2}{c}{{\textbf{Qwen2.5 Trajectory }}}  & \multicolumn{2}{c}{{\textbf{Gemma3 Trajectory }}}\\
    \cmidrule(lr){2-3} \cmidrule(lr){4-5}
    {Annotation}  & {Gemma3} & {GPT-OSS}  & {Gemma3} & {GPT-OSS} \\ 
    \Xhline{1pt}
    \rowcolor{gray!30} {No Error(N.E.)}  & 20.49\% & 16.57\%  & 12.62\% & 13.02\% \\
    {Illegal Action(I.A.)}  & 68.10\% & 68.67\% & 77.22\% & 76.82\%\\
    \rowcolor{gray!30} {Repeated Action(R.A.)}  & 8.00\%  & 7.64\% & 8.16\%  & 8.09\%\\
    {Incorrect Target(I.T.)}  &0.59\%  & 0.59\% & 0.38\%  & 0.69\%  \\
    \rowcolor{gray!30} {Precondition Not Met(P.M.)}  & 0.78\%  & 0.86\% & 0.74\%  & 0.48\% \\
    {Condition Met, Action Not Taken(C.M.)}  & 2.04\%  & 5.07\%  & 0.87\%  & 0.91\% \\
    \Xhline{1.2pt}
\end{tabular}
}
\label{tab:annotation_robustness}
\end{table*}

\subsection{Error Detection}
\label{subsec:error_detect}
In Section~\ref{subsec:dectection}, we provide the detection results across four benchmarks. To facilitate a more comprehensive comparison of performance across each benchmark and LLM backbone, we present results in Table~\ref{tab:extended_detection}. We observe that our method consistently outperforms all baselines across all four benchmarks and three LLM Agents, demonstrating the broad effectiveness and generalizability of our approach.
\begin{table*}[h]
\caption{Extended action error detection marginal average results(\%), the metric is the average accuracy($\uparrow$). GPT4o. is short for GPT4o-mini. The best performances are in \textcolor{red}{\textbf{bold}}, and the second-best method is \textcolor{blue}{\underline{underlined}}.}
\label{tab:extended_detection}
\centering
\scalebox{0.90}{ \begin{tabular}{lccccccc} 
    \Xhline{1.2pt}
    Method  & AlfWorld & TextWorld & ScienceWorld & TravelPlanner & GPT4o. & Qwen2.5 & Gemma3  \\
    \Xhline{1pt}
    \rowcolor{gray!30} \multicolumn{8}{c}{\textbf{Text Classification}} \\
    TF-IDF & \textcolor{blue}{\underline{62.92}}	&	51.86	&	80.73	&	\textcolor{blue}{\underline{74.60}}	&	\textcolor{blue}{\underline{74.02}}	&	\textcolor{blue}{\underline{65.67}}	&	62.90 \\
    Bert   & 57.00	&	\textcolor{blue}{\underline{55.91}}	&	\textcolor{blue}{\underline{80.90}}	&	74.18	&	73.23	&	65.61	&	62.16 \\ 
    \rowcolor{gray!30} \multicolumn{8}{c}{\textbf{Retrieval}}\\
    MiniLM & 45.27	&	48.43	&	74.73	&	69.85	&	65.69	&	56.86	&	56.17 \\
    E5     & 51.82	&	45.61	&	74.68	&	69.44	&	67.34	&	54.36	&	59.47 \\
    GTR    & 53.50	&	48.11	&	74.97	&	69.72	&	68.64	&	55.85	&	60.24 \\
    \rowcolor{gray!30} \multicolumn{8}{c}{\textbf{RAG}} \\
    GPT4o  & 59.43	&	57.02	&	71.71	&	64.14	&	64.99	&	60.72	&	\textcolor{blue}{\underline{63.52}} \\
    Gemma3 & 60.04	&	57.70	&	78.15	&	38.01	&	57.37	&	59.06	&	59.00 \\
    Qwen2.5 & 54.31	&	43.55	&	68.43	&	28.92	&	46.34	&	48.92	&	51.15 \\
    \rowcolor{gray!30} \multicolumn{8}{c}{\textbf{LLM Zero-Shot}} \\
    GPT4o  & 31.44	&	45.46	&	33.36	&	6.97	&	28.86	&	29.10	&	29.96 \\
    Gemma3 & 31.26	&	42.40	&	24.51	&	3.64	&	25.13	&	25.60	&	25.62 \\
    Qwen2.5& 17.82	&	33.02	&	20.13	&	8.19	&	19.53	&	19.67	&	20.17 \\
    \rowcolor{gray!30} \multicolumn{8}{c}{\textbf{LLM One-Shot}} \\
    GPT4o  &24.12	&	48.17	&	12.65	&	7.06	&	22.89	&	20.40	&	25.72 \\
    Gemma3 & 26.89	&	37.64	&	15.20	&	3.73	&	20.52	&	19.90	&	22.19 \\
    Qwen2.5& 17.25	&	37.48	&	23.51	&	8.44	&	22.78	&	19.97	&	22.26 \\
    \rowcolor{gray!30} \multicolumn{8}{c}{\textbf{LLM Three-Shot}} \\
    GPT4o  & 27.76	&	45.09	&	14.22	&	3.45	&	24.71	&	18.32	&	24.86 \\
    Gemma3 & 27.92	&	35.52	&	33.66	&	5.53	&	25.05	&	29.39	&	22.54 \\
    Qwen2.5& 17.40	&	36.28	&	27.71	&	4.01	&	22.53	&	20.87	&	20.66 \\
    \hline
\rowcolor{gray!30} 
    Ours  & \textcolor{red}{\textbf{66.91}}	&	\textcolor{red}{\textbf{65.29}}	&	\textcolor{red}{\textbf{85.27}}	&	\textcolor{red}{\textbf{76.11}}	&	\textcolor{red}{\textbf{77.89}}	&	\textcolor{red}{\textbf{71.02}}	&	\textcolor{red}{\textbf{71.27}} \\
    \Xhline{1.2pt}
\end{tabular}
}
\end{table*}

\stitle{Extensive Statistical Significance Analysis.} To further assess the robustness of the detection methods under varying conditions, we repeat the experiments five times and report the mean results with error bars on TextWorld, which serves as a representative benchmark for this analysis. As shown in Table~\ref{tab:detection_robustness}, our method consistently outperforms all baselines across runs, with a small variance, confirming the stability of our approach and its alignment with our theoretical analysis.

\begin{table*}[h]
\caption{{Statistical significance analysis of action error detection on TextWorld.The best performances are in \textcolor{red}{\textbf{bold}}, and the second-best method is \textcolor{blue}{\underline{underlined}}.} }
\centering
\scalebox{0.95}{ \begin{tabular}{ccccccc} 
    \Xhline{1.2pt}
    & \multicolumn{3}{c}{{\textbf{Accuracy/Micro F1}}}  & \multicolumn{3}{c}{{\textbf{Macro F1}}}\\
    \cmidrule(lr){2-4} \cmidrule(lr){5-7}
    {Method}  & {GPT4o.} & {Qwen2.5} & {Gemma3}  & {GPT4o.} & {Qwen2.5} & {Gemma3}\\
    \Xhline{1pt}
    \rowcolor{gray!30} \multicolumn{7}{c}{{\textbf{Text Classification}}} \\
    {TF-IDF} & \ms{{0.6365}}{{0.0016}} & \ms{\textcolor{blue}{\underline{0.4996}}}{{0.0041}} & \ms{{0.4045}}{{0.0160}}  & \ms{{0.3399}}{{0.0093}} & \ms{{0.2712}}{{0.0042}} & \ms{{0.2256}}{{0.0074}}\\
    {Bert}   & \ms{\textcolor{blue}{\underline{0.6368}}}{{0.0097}} & \ms{\textcolor{blue}{\underline{0.4996}}}{{0.0501}} & \ms{{0.4345}}{{0.0493}}  & \ms{{0.4177}}{{0.0120}} & \ms{{0.3897}}{{0.0380}} & \ms{{0.3711}}{{0.0283}}\\
    \rowcolor{gray!30} \multicolumn{7}{c}{{\textbf{Retrieve}}}\\
    {GTR}    &  \ms{{0.5322}}{{0.0073}} & \ms{{0.3018}}{{0.0009}} & \ms{{0.5322}}{{0.0073}} & \ms{\textcolor{red}{\textbf{0.4257}}}{{0.0041}} & \ms{{0.2472}}{{0.0056}} & \ms{{0.4257}}{{0.0041}} \\
    \rowcolor{gray!30} \multicolumn{7}{c}{{\textbf{RAG}}} \\
    {Gemma3} &  \ms{{0.6232}}{{0.0030}} & \ms{{0.5183}}{{0.0066}} & \ms{\textcolor{blue}{\underline{0.6124}}}{{0.0072}} & \ms{{0.4217}}{{0.0052}} & \ms{\textcolor{red}{\textbf{0.4649}}}{{0.0036}} & \ms{{0.4952}}{{0.0052}}\\
    \rowcolor{gray!30} \multicolumn{7}{c}{{\textbf{LLM Zero-Shot}}} \\
    {Gemma3} & \ms{{0.4732}}{{0.0030}} & \ms{{0.3427}}{{0.0029}} & \ms{{0.4492}}{{0.0000}} & \ms{{0.2522}}{{0.0000}} & \ms{{0.2462}}{{0.0009}} & \ms{{0.2680}}{{0.0000}}\\
    \rowcolor{gray!30} \multicolumn{7}{c}{{\textbf{LLM One-Shot}}} \\
    {Gemma3} & \ms{{0.3755}}{{0.0755}} & \ms{{0.2620}}{{0.0741}} & \ms{{0.3090}}{{0.0554}} & \ms{{0.2157}}{{0.0408}} & \ms{{0.1824}}{{0.0606}} & \ms{{0.1799}}{{0.0442}}\\
    \rowcolor{gray!30} \multicolumn{7}{c}{{\textbf{LLM Three-Shot}}} \\
    {Gemma3} & \ms{{0.1158}}{{0.0457}} & \ms{{0.3194}}{{0.0505}} & \ms{{0.1876}}{{0.0907}} & \ms{{0.0994}}{{0.0441}} & \ms{{0.1932}}{{0.0612}} & \ms{{0.1474}}{{0.0694}}\\
    \hline
    \rowcolor{gray!30} {Ours} & \ms{\textcolor{red}{\textbf{0.6872}}}{{0.0091}} & \ms{\textcolor{red}{\textbf{0.5810}}}{{0.0133}} & \ms{\textcolor{red}{\textbf{0.6316}}}{{0.0229}}  & \ms{\textcolor{blue}{\underline{0.4226}}}{{0.0160}} & \ms{\textcolor{blue}{\underline{0.4298}}}{{0.0063}} & \ms{\textcolor{red}{\textbf{0.5283}}}{{0.0199}}\\
    \Xhline{1.2pt}
\end{tabular}
}
\label{tab:detection_robustness}
\end{table*}

\subsection{Feedback Evaluation}
\label{subsec:feedback_detail_results}
\stitle{Detailed Results.} In Section~\ref{subsec:dectection}, we provide the average results across four benchmarks and three LLM agents; in this subsection, we present a more granular breakdown. As the PR results in Table~\ref{tab:feedback_PR} show, our method achieves 7 best and 4 second-best results across all configurations, demonstrating the consistent advantage of {\pipeline}. Furthermore, in most cases our method outperforms the vanilla baseline, confirming that incorporating feedback information effectively boosts the agent's task completion performance.
\begin{table*}[!h]
\caption{Action error feedback evaluation results. The metric is PR($\uparrow$). GPT4o. is short for GPT4o-mini. The best performances are in \textcolor{red}{\textbf{bold}}, and the second-best method is \textcolor{blue}{\underline{underlined}}.}
\label{tab:feedback_PR}
\centering
\scalebox{0.75}{ \begin{tabular}{l@{\hspace{0.20cm}}c@{\hspace{0.2cm}}c@{\hspace{0.10cm}}c@{\hspace{0.2cm}}c@{\hspace{0.2cm}}c@{\hspace{0.1cm}}c@{\hspace{0.2cm}}c@{\hspace{0.2cm}}c@{\hspace{0.1cm}}c@{\hspace{0.2cm}}c@{\hspace{0.2cm}}c@{\hspace{0.2cm}}c@{\hspace{0.2cm}}c@{\hspace{0.2cm}}} 
    \Xhline{1.2pt}
    \textbf{}  & \multicolumn{3}{c}{\textbf{AlfWorld}} & \multicolumn{3}{c}{\textbf{TextWorld}} & \multicolumn{3}{c}{\textbf{ScienceWorld}} & \multicolumn{3}{c}{\textbf{TravelPlanner}}\\
    \cmidrule(lr){2-4} \cmidrule(lr){5-7} \cmidrule(lr){8-10} \cmidrule(lr){11-13}
    Method  & GPT4o. & Qwen2.5 & Gemma3 & GPT4o. & Qwen2.5 & Gemma3 & GPT4o. & Qwen2.5 & Gemma3 & GPT4o. & Qwen2.5 & Gemma3 \\
    \Xhline{1pt}
    Vanilla & 18.66	&	\textcolor{blue}{\underline{40.30}}	&	6.72	&	45.00	&	\textcolor{blue}{\underline{65.00}}	&	65.00	&	\textcolor{blue}{\underline{77.26}}	&	75.69	&	86.81	&	79.33	&	54.67	&	55.00 \\
    \rowcolor{gray!30} \multicolumn{13}{c}{\textbf{Text Classification}} \\
    +TF-IDF & \textcolor{blue}{\underline{22.22}}	&	31.34	&	\textcolor{blue}{\underline{12.69}}	&	\textcolor{red}{\textbf{65.00}}	&	57.50	&	65.00	&	71.03	&	77.93	&	86.39  &	\textcolor{blue}{\underline{80.67}}	&	\textcolor{blue}{\underline{60.67}}	&	\textcolor{red}{\textbf{97.00}} \\
    \rowcolor{gray!30} \multicolumn{13}{c}{\textbf{Retrieve}}\\
    +GTR & 15.67	&	32.09	&	8.96	&	37.50	&	\textcolor{red}{\textbf{72.50}}	&	67.50	&	67.32	&	79.86	&	83.33  &	80.00	&	59.33	&	95.33 \\
    \rowcolor{gray!30} \multicolumn{13}{c}{\textbf{RAG}} \\
    +Gemma3 & 20.71	&	37.31	&	2.99	&	50.00	&	\textcolor{red}{\textbf{72.50}}	&	\textcolor{blue}{\underline{70.00}}	&	70.08	&	\textcolor{red}{\textbf{86.11}}	&	\textcolor{blue}{\underline{87.50}}  &	\textcolor{red}{\textbf{82.67}}	&	59.67	&	\textcolor{blue}{\underline{96.33}} \\
    \rowcolor{gray!30} \multicolumn{13}{c}{\textbf{LLM Zero-Shot}} \\
    +GPT4o & 18.66	&	33.58	&	\textcolor{red}{\textbf{17.16}}	&	\textcolor{blue}{\underline{52.50}}	&	\textcolor{blue}{\underline{65.00}}	&	\textcolor{red}{\textbf{72.50}}	&	71.35	&	80.41	&	85.42  &	\textcolor{blue}{\underline{80.67}}	&	57.33	&	92.67 \\
    \hline
    \rowcolor{gray!30} +Ours & \textcolor{red}{\textbf{26.12}}
	&	\textcolor{red}{\textbf{44.78}}	&	\textcolor{blue}{\underline{12.69}}	&	46.51	&	\textcolor{blue}{\underline{65.00}}	&	\textcolor{red}{\textbf{75.00}}	&	\textcolor{red}{\textbf{94.44}}	&	\textcolor{blue}{\underline{82.64}}	&	\textcolor{red}{\textbf{94.44}}	&	\textcolor{blue}{\underline{80.67}}	&	\textcolor{red}{\textbf{67.67}}	&	\textcolor{red}{\textbf{97.00}} \\
    \Xhline{1.2pt}
\end{tabular}
}
\end{table*}

\stitle{Grounding Ratio Analysis.} Regarding the GR, as shown in Table~\ref{tab:feedback_GR}, our method achieves the second-best GR overall. It is important to note that GR only indicates whether an action lies within the valid action space, but does not guarantee its correctness. Compared to the vanilla baseline, all methods exhibit varying GRs, suggesting that the feedback mechanism encourages the agent to explore alternative actions rather than converging on familiar ones. Moreover, GR is jointly influenced by task characteristics and the reasoning ability of the underlying LLM. For instance, in TextWorld, the action space contains many variants, making it inherently harder to generate grounded actions even with feedback; in contrast, TravelPlanner features a more constrained action space, leading to improved GR under feedback. We also observe model-dependent behavior: for GPT4o-mini and Qwen2.5-14B, the GR is lower than the vanilla baseline, whereas Gemma3-27B shows the opposite trend, suggesting that stronger reasoning capabilities allow the model to better leverage feedback for grounded action generation. Importantly, a brute-force guessing strategy would typically yield a similar PR but a lower GR compared to the baseline; however, we do not observe this pattern in our results, confirming that the feedback mechanism does not reduce to random guessing. To further confirm the statistical significance of these findings, we repeat the experiments five times and report results with error bars in Table~\ref{tab:detection_robustness}, where our method consistently outperforms all baselines with low variance.

Another notable observation is that GR does not directly correlate with PR. For instance, in TextWorld, although the GR of all methods falls below the vanilla baseline, the PR still improves. This is expected, as the feedback mechanism does not directly modify the LLM's parameters or decoding behavior, but instead encourages the generation of more diverse candidate actions, even when some contain errors. This diversity ultimately benefits task progress, decoupling PR gains from GR, and highlighting that grounding accuracy alone is insufficient as a proxy for overall agent performance.

\begin{table*}[!h]
\caption{Action error feedback evaluation results. The metric is GR($\uparrow$). GPT4o. is short for GPT4o-mini. The best performances are in \textcolor{red}{\textbf{bold}}, and the second-best method is \textcolor{blue}{\underline{underlined}}.}
\label{tab:feedback_GR}
\centering
\scalebox{0.75}{ \begin{tabular}{l@{\hspace{0.20cm}}c@{\hspace{0.2cm}}c@{\hspace{0.10cm}}c@{\hspace{0.2cm}}c@{\hspace{0.2cm}}c@{\hspace{0.1cm}}c@{\hspace{0.2cm}}c@{\hspace{0.2cm}}c@{\hspace{0.1cm}}c@{\hspace{0.2cm}}c@{\hspace{0.2cm}}c@{\hspace{0.2cm}}c@{\hspace{0.2cm}}c@{\hspace{0.2cm}}} 
    \Xhline{1.2pt}
    \textbf{}  & \multicolumn{3}{c}{\textbf{AlfWorld}} & \multicolumn{3}{c}{\textbf{TextWorld}} & \multicolumn{3}{c}{\textbf{ScienceWorld}} & \multicolumn{3}{c}{\textbf{TravelPlanner}}\\
    \cmidrule(lr){2-4} \cmidrule(lr){5-7} \cmidrule(lr){8-10} \cmidrule(lr){11-13}
    Method  & GPT4o. & Qwen2.5 & Gemma3 & GPT4o. & Qwen2.5 & Gemma3 & GPT4o. & Qwen2.5 & Gemma3 & GPT4o. & Qwen2.5 & Gemma3 \\
    \Xhline{1pt}
    Vanilla & 73.87	&	\textcolor{red}{\textbf{68.02}}	&	62.49	&	\textcolor{red}{\textbf{95.66}}	&	\textcolor{red}{\textbf{86.79}}	&	\textcolor{red}{\textbf{91.07}}	&	\textcolor{red}{\textbf{77.35}}	&	49.62	&	70.22	&	97.71	&	74.14	&	74.20\\
    \rowcolor{gray!30} \multicolumn{13}{c}{\textbf{Text Classification}} \\
    +TF-IDF & 74.22	&	58.12	&	63.04	&	53.33	&	46.07	&	48.02	&	71.03	&	\textcolor{red}{\textbf{58.69}}	&	\textcolor{red}{\textbf{75.03}}	&	97.83	&	75.58	&	\textcolor{blue}{\underline{99.58}} \\
    \rowcolor{gray!30} \multicolumn{13}{c}{\textbf{Retrieve}}\\
    +GTR & 74.73	&	54.89	&	64.66	&	58.11	&	43.84	&	52.58	&	67.32	&	55.42	&	72.16	&	97.73	&	\textcolor{blue}{\underline{76.48}}	&	99.21\\
    \rowcolor{gray!30} \multicolumn{13}{c}{\textbf{RAG}} \\
    +Gemma3 & \textcolor{blue}{\underline{75.78}}	&	61.88	&	\textcolor{red}{\textbf{91.19}}	&	57.35	&	44.24	&	52.33	&	70.08	&	\textcolor{blue}{\underline{57.58}}	&	71.41	&	\textcolor{blue}{\underline{97.83}}	&	76.04	&	99.53\\
    \rowcolor{gray!30} \multicolumn{13}{c}{\textbf{LLM Zero-Shot}} \\
    +GPT4o &  \textcolor{red}{\textbf{75.84}}	&	57.55	&	65.88	&	55.20	&	46.57	&	52.82	&	\textcolor{blue}{\underline{71.35}}	&	52.12	&	\textcolor{blue}{\underline{74.12}}	&	97.76	&	75.36	&	98.52  \\
    \hline
    \rowcolor{gray!30} +Ours & 71.19	&	\textcolor{blue}{\underline{62.33}}	&	\textcolor{blue}{\underline{66.97}}	&	\textcolor{blue}{\underline{67.89}}	&	\textcolor{blue}{\underline{62.78}}	&	\textcolor{blue}{\underline{89.79}}	&	69.09	&	57.26	&	73.55	&	\textcolor{red}{\textbf{98.07}}	&	\textcolor{red}{\textbf{81.23}}	&	\textcolor{red}{\textbf{99.81}}\\
    \Xhline{1.2pt}
\end{tabular}
}
\end{table*}

\stitle{Detailed Results of Deliberate Prompting-based LLM agents.}  We further provide the PR and GR in Table~\ref{tab:comparison_two}. As mentioned in Section~\ref{subsec:feedback}, our method mitigates these limitations, which highly rely on  reasoning capabilities of the LLM, by incorporating structured external feedback, allowing the model to self-correct more reliably without depending solely on its parametric knowledge. As a result, our approach remains robust even when applied to smaller backbone models.
\begin{table}[h]
\caption{{Comparison between ours and modern prompts on TextWorld. The best performances are in \textcolor{red}{\textbf{bold}}.} }
\centering
   \renewcommand\tabcolsep{3.5pt}
\scalebox{1.0}{ \begin{tabular}{ccccc} 
    \Xhline{1.2pt}
    & \multicolumn{2}{c}{{\textbf{Pass Ratio(\%)}}}  & \multicolumn{2}{c}{{\textbf{Ground Ratio(\%)}}}\\
    \cmidrule(lr){2-3} \cmidrule(lr){4-5}
    {Method}  & {Qwen2.5} & {Gemma3}  & {Qwen2.5} & {Gemma3}\\
    \Xhline{1pt}
    \rowcolor{gray!20}  {ReAct}      & {62.50}  & {70.00}   & {69.86}  & {86.60} \\
    {Reflexion} & {60.00}  & {67.50}   & \textcolor{red}{\textbf{71.18}}  & {79.60} \\
    \rowcolor{gray!20}  {Ours}       & \textcolor{red}{\textbf{65.00}}  & \textcolor{red}{\textbf{75.00}}   & {62.78}  & \textcolor{red}{\textbf{89.78}}  \\
    \Xhline{1.2pt}
\end{tabular}
}
\label{tab:comparison_two}
\end{table}

\begin{table}[!h]
\caption{{Statistical significance analysis of error feedback on TextWorld.The best performances are in \textcolor{red}{\textbf{bold}}.} }
\centering
\renewcommand\tabcolsep{3.5pt}
\scalebox{0.95}{ \begin{tabular}{ccccc} 
    \Xhline{1.2pt}
    & \multicolumn{2}{c}{{\textbf{Pass Ratio(\%)}}}  & \multicolumn{2}{c}{{\textbf{Ground Ratio(\%)}}}\\
    \cmidrule(lr){2-3} \cmidrule(lr){4-5}
    {Method}  & {Qwen2.5} & {Gemma3}  & {Qwen2.5} & {Gemma3}\\
    \Xhline{1pt}
    {Vanilla}  & \ms{{65.00}}{{6.52}} & \ms{{70.50}}{{5.50}} & \ms{\textcolor{red}{\textbf{88.77}}}{{2.98}} & \ms{\textcolor{red}{\textbf{91.49}}}{{3.42}} \\
    \rowcolor{gray!20} \multicolumn{5}{c}{{\textbf{Text Classification}}} \\
    {+TF-IDF}  & \ms{{61.50}}{{3.35}}  & \ms{{67.50}}{{7.07}}  & \ms{{40.49}}{{2.35}} & \ms{{46.81}}{{1.51}}\\
    \rowcolor{gray!10} \multicolumn{5}{c}{{\textbf{Retrieve}}}\\
    {+GTR}    & \ms{{67.50}}{{6.12}} & \ms{{73.50}}{{6.75}}  & \ms{{41.16}}{{1.73}} & \ms{{51.52}}{{1.13}} \\
    \rowcolor{gray!20} \multicolumn{5}{c}{{\textbf{RAG}}} \\
    {+Gemma3}  & \ms{\textcolor{red}{\textbf{69.50}}}{{3.71}}  & \ms{{71.94}}{{5.36}}   & \ms{{40.81}}{{0.82}} & \ms{{52.66}}{{1.42}} \\
    \hline
    \rowcolor{gray!20} {+Ours} & \ms{{64.00}}{{1.22}} & \ms{\textcolor{red}{\textbf{76.00}}}{{3.74}} & \ms{{73.61}}{{6.74}} & \ms{{87.15}}{{4.78}} \\
    \Xhline{1.2pt}
\end{tabular}
}
\label{tab:improve_robustness}
\end{table}

\stitle{Statistical Significance Analysis.} To conduct the variance analysis, we perform experiments reporting error bars for feedback-improvement evaluations on the TextWorld environment. We repeat each experiment five times and report the mean and standard deviation. The results are presented in Table~\ref{tab:improve_robustness}. As the results demonstrate, our method achieves robust and stable performance across all runs, with low variance pass ratio consistently aligned with the findings reported in previous experiments. This confirms that our method is not sensitive to random initialization or stochastic variation, further validating the reliability of our reported results.

\subsection{Ablation Study}
\label{subsec:ablation_detail_results}
\stitle{Graph Building.} 
In Section~\ref{sec:ablation}, we provide an analysis of two types of graphs and node embedding initialization. In this section, we further analyze the distinction between action-centric and observation-centric graph designs. As shown in Table~\ref{tab:graph_building}, in ALFWorld, when modeling actions as nodes, the resulting graph scale is smaller than when using observations as nodes. However, even with Qwen2.5-14B, the number of unique actions exceeds that of observations, while the number of edges remains smaller. This suggests that an action-centric graph yields a more compact yet informative structure.

From the perspective of graph construction, the observation space is significantly larger than the action space. Treating observations as nodes would therefore greatly increase the graph size and complexity, imposing additional computational overhead and making it harder for the model to efficiently capture task-relevant relationships. Using observations as nodes and actions as edges tends to produce a much sparser graph, which may limit the model's ability to learn meaningful representations. In contrast, our design, where actions serve as nodes, encourages denser connectivity and richer structural information, enabling the model to better capture the sequential and relational nature of agent decision-making.

\begin{table*}[!h]
\caption{{The number of nodes/edges between action-centric/observation-centric graphs of the AlfWorld environment.} }
\centering
\scalebox{1.0}{ \begin{tabular}{ccccc} 
    \Xhline{1.2pt}
      & {GPT4o-mini} & {Qwen2.5-14B}  & {Gemma3-27B} \\
    \Xhline{1pt}
    {node(action)/edge(observation)}  & 682/2,900  & 1,825/3,868   & 623/1,983 \\
    {node(observation)/edge(action)}  & 1,154/5,938  & 910/4,404   & 834/4,064 \\
    \Xhline{1.2pt}
\end{tabular}
}
\label{tab:graph_building}
\end{table*}

\stitle{Detailed Results of Prompt Ablation Study.} We provide the detailed results of the feedback integration of PR and GR in Table~\ref{tab:ablation_form}.
\begin{table}[!h]
\caption{{Ablation study of feedback integration mechanisms on TextWorld. In the table, ``exp.'' is short for ``explanations''. The best performances are in \textcolor{red}{\textbf{bold}}.} }
\centering
   \renewcommand\tabcolsep{2.5pt}
\scalebox{1.0}{ \begin{tabular}{ccccc} 
    \Xhline{1.2pt}
    & \multicolumn{2}{c}{{\textbf{Pass Ratio(\%)}}}  & \multicolumn{2}{c}{{\textbf{Ground Ratio(\%)}}}\\
    \cmidrule(lr){2-3} \cmidrule(lr){4-5}
    {Method}  & {Qwen2.5} & {Gemma3}  & {Qwen2.5} & {Gemma3}\\
    \Xhline{1pt}
    \rowcolor{gray!20}  {in trajectories}  & {57.50}  & {70.00}   & {52.49} & {86.69} \\
    {in system}        & {60.00}  & \textcolor{red}{\textbf{75.00}}   & {57.41} & {88.10} \\
    \rowcolor{gray!20}  {in user}          & \textcolor{red}{\textbf{65.00}}  & \textcolor{red}{\textbf{75.00}}   & {62.78}  & \textcolor{red}{\textbf{89.79}} \\
    {in user w/o exp.} & {62.50}  & {72.50}   & \textcolor{red}{\textbf{71.24}}  & {84.79} \\
    \Xhline{1.2pt}
\end{tabular}
}
\label{tab:ablation_form}
\end{table}

\stitle{Detailed Hyperparameter Study of Feedback.} In Section~\ref{sec:ablation}, we examine the impact of two key hyperparameters in {\pipeline}: the maximum number of attempts and the confidence threshold. We provide the detailed results in Table~\ref{tab:detaileded_ablation}. In this section, we further investigate the role of error samples in enhancing performance. Table~\ref{tab:feedback_error_samples} summarizes the results under two settings: with and without error samples. Our findings show that, without error samples, increasing the number of attempts leads to consistent improvements in PR performance, while GR remains largely unchanged. When error samples are included, PR performance also improves; however, adding more samples does not yield further gains and, in fact, leads to a decline in GR. This suggests a trade-off: while error samples provide useful information for correction, an excessive number of them may introduce noise or bias, ultimately hindering generalization.

\begin{table*}[t]
\caption{Ablation study on confidence threshold and maximum attempts of feedback evaluation.}
\label{tab:detaileded_ablation}
\centering
\scalebox{0.83}{ \begin{tabular}{lcccccccc} 
    \Xhline{1.2pt}
     & \multicolumn{2}{c}{\textbf{TextWorld}(PR)} & \multicolumn{2}{c}{\textbf{ScienceWorld}(PR)} & \multicolumn{2}{c}{\textbf{TextWorld}(GR)} & \multicolumn{2}{c}{\textbf{ScienceWorld}(GR)} \\
    \cmidrule(lr){2-3}\cmidrule(lr){4-5} \cmidrule(lr){6-7} \cmidrule(lr){8-9}
    Hyperparameter & Qwen2.5 & Gemma3 & Qwen2.5 & Gemma3 & Qwen2.5 & Gemma3 & Qwen2.5 & Gemma3 \\
    \Xhline{1pt}
    Vanilla & 65.00	&	65.00	&	75.69	&	86.81 & 86.79	&	91.07	&	49.62	&	70.22 \\
    \rowcolor{gray!30} \multicolumn{9}{c}{\textbf{Threshold-0.6 }} \\
    + 1 Attempts 	&	65.00	&	65.00			&	75.69	&	86.81 &	86.79	&	91.07			&	49.62	&	70.22  \\
    + 2 Attempts 	&   67.50	&	72.50			&	96.81	&	94.44 &	65.05	&	93.09			&	60.59	&	71.67  \\
    + 3 Attempts 	&	65.00	&	75.00			&	82.64	&	94.44 &	62.78	&	89.79			&	57.26	&	73.55 \\ 
    + 5 Attempts 	&	62.50	&	70.00			&	82.64	&	95.92 &	58.02	&	90.56			&	57.51	&	73.71 \\ 
    + 7 Attempts 	&	70.00	&	72.50			&	86.11	&	93.75 & 64.71	&	91.71			&	59.91	&	73.47 \\ 
    \rowcolor{gray!30} \multicolumn{9}{c}{\textbf{3 Attempts}} \\
    + Threshold-0.20 & 62.50	&	67.50	&	77.78	&	96.53 &	60.10	&	90.80 &	56.59	&	75.32\\
    + Threshold-0.35 & 65.00	&	75.00 	&	88.36	&	96.53 &	58.16	&	92.46 &	57.91	&	72.91  \\
    + Threshold-0.50 & 62.50	&	77.50   &	87.50	&	93.06 &	55.16	&	92.20 &	58.82	&	73.36 \\ 
    + Threshold-0.65 & 57.50	&	75.00   &	86.11	&	94.44 &	58.31	&	89.11 &	61.28	&	74.72 \\ 
    + Threshold-0.80 & 65.00	&	72.50	&	88.19	&	86.11 &	63.36	&	92.43 &	58.74	&	74.91  \\ 
    \Xhline{1.2pt}
\end{tabular}
}
\end{table*}

\begin{table*}[t]
\caption{Ablation study on error samples of feedback evaluation.}
\label{tab:feedback_error_samples}
\centering
\scalebox{0.83}{ \begin{tabular}{lcccccccc} 
    \Xhline{1.2pt}
     & \multicolumn{2}{c}{\textbf{TextWorld}(PR)} & \multicolumn{2}{c}{\textbf{ScienceWorld}(PR)} & \multicolumn{2}{c}{\textbf{TextWorld}(GR)} & \multicolumn{2}{c}{\textbf{ScienceWorld}(GR)} \\
    \cmidrule(lr){2-3}\cmidrule(lr){4-5} \cmidrule(lr){6-7} \cmidrule(lr){8-9}
    Hyperparameter & Qwen2.5 & Gemma3 & Qwen2.5 & Gemma3 & Qwen2.5 & Gemma3 & Qwen2.5 & Gemma3 \\
    \Xhline{1pt}
    Vanilla & 65.00	&	65.00	&	75.69	&	86.81 & 86.79	&	91.07	&	49.62	&	70.22 \\
    \rowcolor{gray!30} \multicolumn{9}{c}{\textbf{Zero-Shot + Threshold-0.6 }} \\
    + 2 Attempts & 60.00	&	72.50			&	88.19	&	93.06 &	71.07	&	90.50			&	66.28	&	72.51 \\
    + 3 Attempts & 50.00	&	77.50			&	87.50	&	94.44 &	69.20	&	89.28			&	63.64	&	73.86 \\
    + 5 Attempts & 50.00	&	70.00			&	91.67	&	91.67  &	67.97	&	92.16			&	66.28	&	74.49 \\ 
    \rowcolor{gray!30} \multicolumn{9}{c}{\textbf{3 Attempts + Threshold-0.6 }} \\
    + One Shot & 60.00	&	72.50			&	94.44	&	97.22 &	71.19	&	88.55			&	65.87	&	74.20 \\
    + Three Shots     & 65.00	&	75.00			&	82.64	&	94.44 & 62.78	&	89.79			&	57.26	&	73.55 \\
    + Five Shots    & 65.00	&	77.50			&	82.64	&	93.06  & 55.82	&	92.01			&	56.21	&	73.29 \\
    \Xhline{1.2pt}
\end{tabular}
}
\end{table*}

\stitle{Noise Injection.} To evaluate the robustness of our method, we conduct ablation experiments on the feedback component under varying levels of input perturbation. Specifically, we inject random noise into the detection results at noise intensities of {0.0, 0.2, 0.4, 0.6, 0.8}. The results are presented in Figure~\ref{fig:ablation_random}. As the noise intensity increases, the model's performance gradually decreases, which is expected given that higher noise levels corrupt the quality of the feedback signal. Nevertheless, our method maintains competitive performance at moderate noise levels, demonstrating that it is not overly sensitive to minor perturbations in the feedback. This highlights the resilience of our approach and suggests that it can generalize well even when the feedback source is imperfect or noisy.

\begin{figure}[!h]
\centering
\includegraphics[width=0.9\textwidth]{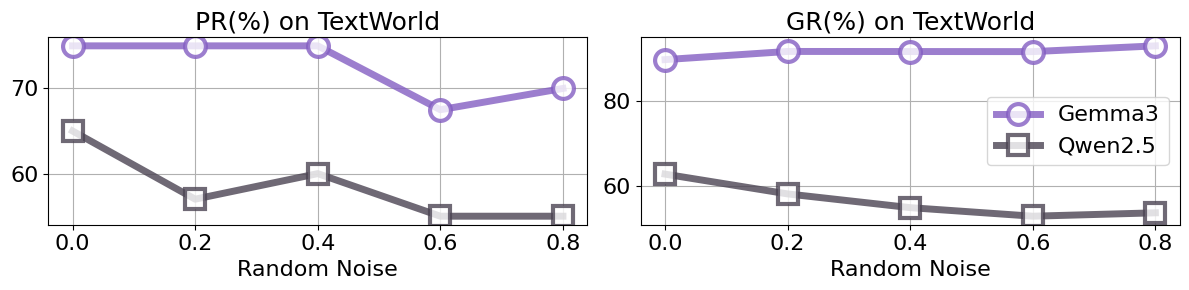}
\caption{{The pass ratio~(\%) and ground ratio~(\%) results on random noise injection.}}
\label{fig:ablation_random}
\vspace{-1.em}
\end{figure}

\subsection{{Prompts}}
{In this subsection, we provide the prompts used in experiments.}
\begin{tcolorbox}[notitle, sharp corners, breakable, colframe=Periwinkle, colback=white, 
       boxrule=3pt, boxsep=0.5pt, enhanced, 
       shadow={3pt}{-3pt}{0pt}{opacity=1,mygrey},
       title={LLM Prompt for AlfWorld, TextWorld, ScienceWorld Action Error Label},]
       \label{box:dataset_prompt_embodiedAI}
       \footnotesize
       {\fontfamily{pcr}\selectfont
            \begin{lstlisting}
"""
You are a top-level expert AI agent trajectory analyst. Your 
primary task is to analyze the 'CURRENT STEP' of an agent's 
trajectory and identify the single most likely error.

**Crucial Constraint**: The agent is unfamiliar with the 
environment, so it might take some time to explore and find 
destinations and objects. Therefore, its actions should be judged 
based on the **Overall Task Goal**, not by strictly following 
the expert plan. The expert plan is only a reference for a possible 
successful path. You can know how much the agent has inferred the 
environment from the observation and extract the best possible 
action from the observation. 

**Output Format**: Your entire response MUST be a single, valid JSON 
object. Do not include any text, explanations, or markdown 
formatting before or after the JSON object. 
The JSON object must contain exactly three keys:
- "has_error": (boolean) true if there is an error(No Error is not
an error), otherwise false.
- "error_type": (string) One of the specified error types below.
- "analysis": (string) A concise, one-sentence explanation for the 
error.For example, for a "Precondition Not Met" error, a good 
analysis would be "The agent tried to place an object that it was 
not holding in its inventory."

**Types**:
- "No Error": The action is logical and contributes to the task 
goal.
- "Precondition Not Met": The action is valid but cannot be gain 
progress because a necessary prior condition is not met (e.g., 
trying to 'put' an object that is not currently held).
- "Condition Met, Action Not Taken": A critical step was available, 
necessary to progress and the agent already observed the condition, 
but the agent performed an irrelevant or less optimal action 
instead (e.g., finding the target object but not picking it up).
- "Incorrect Target": The action is performed on the wrong object or 
at the wrong location (e.g., picking up a 'cloth' instead of the 
required 'soapbottle').
- "Repeated Action": The action is part of a sequence that was 
already performed and without any progress.

--- CONTEXT ---
1.  **Overall Task Goal**:
    "{task_goal}"
2.  **Expert's Suggested Plan (for reference only, may be flawed)**:
    {expert_plan}
3.  **Trajectory History (Recent Steps)**:
    {trajectory_history}
4.  **CURRENT STEP TO ANALYZE**:
    - Action: "{current_action}"

--- YOUR RESPONSE (JSON ONLY) ---
"""
            \end{lstlisting}
            }
\end{tcolorbox}

\begin{tcolorbox}[notitle, sharp corners, breakable, colframe=Periwinkle, colback=white, 
       boxrule=3pt, boxsep=0.5pt, enhanced, 
       shadow={3pt}{-3pt}{0pt}{opacity=1,mygrey},
       title={LLM Prompt for TravelPlanner Action Error Label},]
       \label{box:dataset_prompt_travelplanner}
       \footnotesize
       {\fontfamily{pcr}\selectfont
            \begin{lstlisting}
prompt = f"""Please analyze the following action in a travel 
planning trajectory and classify it into one of these error 
categories. Before giving a conclusion, please read the Agent 
thought(indicate the next step) and observation, reason the 
state of the agent. In general, the action should be consistent 
with the thought.

ERROR CATEGORIES:
1. No error: There is no error during the trajectory
2. Illegal Action: The action is not a valid action
3. Repeated Action: The action has been done in the trajectory, and 
there is nothing updated. It is not necessary.
4. Incorrect Target: The action is not aligned with the goal, for 
example the thought would go to place A, but the action searches 
place B.
5. Precondition Not Met: The action is valid, but it can't be done 
right now, because some condition is not met. For example, the 
current information does not contain the hotel information, so it 
is not a good time to make final plan. Or the current information 
does not have a valid plan due to the constraints, such as money.
6. Precondition Met, Action Not Taken: The information needed is 
collected. The agent can have a valid plan according to the 
information, but it did not make the final plan immediately.

TASK DESCRIPTION:
{current_step.get('task_description', 'N/A')}

TRAJECTORY CONTEXT:
{context_str}

CURRENT STEP TO ANALYZE:
Agent Thought: {current_step.get('thought', 'N/A')}
Action: {current_step.get('action', 'N/A')}

Based on the context and current step, return strict JSON with 
two keys: label and reason. The label must be exactly one of the 
categories above (e.g., "No error", "Illegal Action", etc.). The 
reason should be a brief phrase explaining why.

Attention: Only the raw json."""
            \end{lstlisting}
            }
\end{tcolorbox}

\begin{tcolorbox}[notitle, sharp corners, breakable, colframe=Periwinkle, colback=white, 
       boxrule=3pt, boxsep=0.5pt, enhanced, 
       shadow={3pt}{-3pt}{0pt}{opacity=1,mygrey},
       title={Agent Prompt for AlfWorld environment},]
       \label{box:agent_prompt_alfworld}
       \footnotesize
       {\fontfamily{pcr}\selectfont
            \begin{lstlisting}
messages = [{
"role": "system",
"content": f"""You are an agent in a text-based ALFWorld environment, 
performing a household task.\n
For each step, generate one action based on the task description, 
action Options, current observation, and the plan.\n
Please think about the environmental state from the historical 
trajectories before you give the answer.\n
Do not generate other text except the action itself. 
One action per step.\n
The action output format must be ##Action: XXX ##, 
where XXX is the action.\n\n
Example:\n> Observation: You are in the middle of a room. 
Looking quickly around you, you see a bathtubbasin 1, a cabinet 2, 
a cabinet 1, a countertop 1, a garbagecan 1, a handtowelholder 1, 
a sinkbasin 1, a toilet 1, a toiletpaperhanger 1, and 
a towelholder 1.\n
Your task is to: put a toiletpaper in toiletpaperhanger.\n
> Action: go to toiletpaperhanger 1\n
> Observation: On the toiletpaperhanger 1, you see nothing.\n
> Action: go to toilet 1\n> Observation: On the toilet 1, you see a
soapbottle 1, and a toiletpaper 1.\n
> Action: take toiletpaper 1 from toilet 1\n> Observation: You pick
up the toiletpaper 1 from the toilet 1.\n
> Action: go to toiletpaperhanger 1\n> Observation: On the 
toiletpaperhanger 1, you see nothing.\n
> Action: put toiletpaper 1 in/on toiletpaperhanger 1\n"""},
{"role": "user",
"content": f"""
> Task Description: {task_description}
> Action Options: {','.join(action_lists)}
> Trajectory History: {' '.join(history)}
> Action:"""} ]
            \end{lstlisting}
            }
\end{tcolorbox}

\begin{tcolorbox}[notitle, sharp corners, breakable, colframe=Periwinkle, colback=white, 
       boxrule=3pt, boxsep=0.5pt, enhanced, 
       shadow={3pt}{-3pt}{0pt}{opacity=1,mygrey},
       title={Agent Prompt for TextWorld environment},]
       \label{box:agent_prompt_textworld}
       \footnotesize
       {\fontfamily{pcr}\selectfont
            \begin{lstlisting}
messages = [{
"role": "system",
"content": f"""You are an agent in a text-based TextWorld 
environment, performing a task.\n
For each step, generate one action based on the task description, 
action Options, current observation, and the plan.\n
Please think the environment state from the history trajectories 
before you give the answer.\n
Do not generate other text except the action itself. One action 
per step.\n
The action output format must be ##Action: XXX ##, where XXX is the
action.\n\n
Example:\n
> Observation: You are in a room. You see a coin, a table, 
and a door to the north (open).\n
Your task is to: find and collect a coin.\n
> Action: look\n
> Observation: You see a coin.\n
> Action: take coin\n
> Observation: You take the coin.\n"""},
{"role": "user",
"content": f"""
> Task Description: {task_description}
> Action Options: {','.join(action_list)}
> Trajectory History: {' '.join(history)}
> Action:""" }]
            \end{lstlisting}
            }
\end{tcolorbox}

\begin{tcolorbox}[notitle, sharp corners, breakable, colframe=Periwinkle, colback=white, 
       boxrule=3pt, boxsep=0.5pt, enhanced, 
       shadow={3pt}{-3pt}{0pt}{opacity=1,mygrey},
       title={Agent Prompt for ScienceWorld environment},]
       \label{box:agent_prompt_scienceworld}
       \footnotesize
       {\fontfamily{pcr}\selectfont
            \begin{lstlisting}
 messages = [ {
role="system",
content=( "You are an agent in a text-based ScienceWorld 
environment, performing a task.\n"
"For each step, generate one action based on the task description, 
action Options, current observation, and the plan.\n"
"Please think the environment state from the history trajectories 
before you give the answer.\n"
"Possible Available Actions contains the OBJ, which can be replace 
by object in environment.\n"
"Do not generate other text except the action itself. One action 
per step.\n"
"The action output format must be ##Action: XXX ##, where XXX is 
the action.\n\n"
"Example:\n
> Observation: This room is called the hallway. In it, you see: a 
picture, a substance called air, the agent.\n"
"You also see: doors to other rooms.\nYour task is to: boil water.\n
> Action: look around\n
> Observation: The door is already open.\n
> Action: open door to kitchen\n" ) },
{ role="user",
content=(
f"> Task Description: {task_description}\n"
f"> Action Options: {','.join(action_lists)}\n"
f"> Trajectory History: {' '.join(history)}\n"
f"> Action:\n" ) } ]
            \end{lstlisting}
            }
\end{tcolorbox}

\begin{tcolorbox}[notitle, sharp corners, breakable, colframe=Periwinkle, colback=white, 
       boxrule=3pt, boxsep=0.5pt, enhanced, 
       shadow={3pt}{-3pt}{0pt}{opacity=1,mygrey},
       title={LLM Prompt with the Error Feedback},]
       \label{box:prompt_improvement}
       \footnotesize
       {\fontfamily{pcr}\selectfont
            \begin{lstlisting}
feedback_section = f"\n\nIMPORTANT: We have analyzed 
your last time action {action} for the same trajectory and provided
this feedback: {gnn_feedback}\n Do not generate the same action,
{action}, again.\n"

messages = [{...},
{"role": "user",
"content": f"""
> Task Description: {task_description}
> Action Options: {','.join(action_lists)}
> Trajectory History: {' '.join(history)}
> Action:
{feedback_section} """}]
            \end{lstlisting}
            }
\end{tcolorbox}
\begin{tcolorbox}[notitle, sharp corners, breakable, colframe=Periwinkle, colback=white, 
       boxrule=3pt, boxsep=0.5pt, enhanced, 
       shadow={3pt}{-3pt}{0pt}{opacity=1,mygrey},
       title={LLM-as-Judge prompt for Error Detection},]
       \label{box:prompt_llm_as_judge_detection}
       \footnotesize
       {\fontfamily{pcr}\selectfont
            \begin{lstlisting}
prompt = """You are a helpful assistant for classifying 
agent actions.Please give the label for the following 
input. Here are labels:\n

- No Error: The action is logical and contributes to the task 
goal.
- Illegal Action: The action performed an illegal action or has 
no effect.
- Precondition Not Met: The action is valid but cannot be gain 
progress because a necessary prior condition is not met (e.g., 
trying to 'put' an object that is not currently held).
- Condition Met, Action Not Taken: A critical step was available, 
necessary to progress and the agent already observed the condition, 
but the agent performed an irrelevant or less optimal action 
instead (e.g., finding the target object but not picking it up).
- Incorrect Target: The action is performed on the wrong object or 
at the wrong location (e.g., picking up a 'cloth' instead of the 
required 'soapbottle').
- Repeated Action: The action is part of a sequence that was 
already performed and without any progress.\n """

for shot in shots:
    prompt += f"Example:\nInput: {shot['input']}\n
                Label: {shot['label']}\n\n"
prompt += f"Input: {input_text}\nLabel:"

            \end{lstlisting}
            }
\end{tcolorbox}

\begin{tcolorbox}[notitle, sharp corners, breakable, colframe=Periwinkle, colback=white, 
       boxrule=3pt, boxsep=0.5pt, enhanced, 
       shadow={3pt}{-3pt}{0pt}{opacity=1,mygrey},
       title={RAG prompt for Error Detection},]
       \label{box:prompt_rag_detection}
       \footnotesize
       {\fontfamily{pcr}\selectfont
            \begin{lstlisting}
prompt = "You are a classifier. Given the following examples, 
predict the label for the new sample.\n"
for idx, (txt, lbl) in enumerate(retrieved):
    prompt += f"Example {idx+1}:\nText: {txt}\nLabel: {lbl}\n"
prompt += f"\nNow,classify this sample:\nText: {X_test[i]}\nLabel: "
            \end{lstlisting}
            }
\end{tcolorbox}

\end{document}